\documentclass[10pt,onecolumn]{IEEEtran}

\usepackage[english]{babel}
\usepackage[latin1]{inputenc}
\usepackage{enumerate}
\usepackage[usenames, dvipsnames]{color}
\usepackage[T1]{fontenc}
\usepackage{subfigure}
\usepackage{dsfont}
\usepackage{amsfonts}
\usepackage[final]{graphicx}
\usepackage{amsmath}
\usepackage{mathtools, cuted}

\usepackage{amsthm}
\usepackage{amstext}
\usepackage{amssymb}
\usepackage{mathrsfs}
\usepackage{cite}
\usepackage{breqn}
\usepackage{mathtools}
\usepackage{tikz}
\usepackage{marginnote}
\usepackage{mdframed}
\usepackage{tkz-tab}
\usetikzlibrary{topaths,calc}
\usepackage{pgfplots}
\usepackage[ruled]{algorithm2e}
\usepackage{tabularx}
\usepackage{stackengine}

\allowdisplaybreaks

\usepackage{xargs} 
\usepackage[colorinlistoftodos,prependcaption,textsize=footnotesize]{todonotes}
\newcommandx{\unsure}[2][1=]{\todo[linecolor=red,backgroundcolor=red!25,bordercolor=red,#1]{#2}}
\newcommandx{\change}[2][1=]{\todo[linecolor=blue,backgroundcolor=blue!25,bordercolor=blue,#1]{#2}}
\newcommandx{\info}[2][1=]{\todo[linecolor=OliveGreen,backgroundcolor=OliveGreen!25,bordercolor=OliveGreen,#1]{#2}}
\newcommandx{\improvement}[2][1=]{\todo[linecolor=Plum,backgroundcolor=Plum!25,bordercolor=Plum,#1]{#2}}
\newcommandx{\thiswillnotshow}[2][1=]{\todo[disable,#1]{#2}}

\usetikzlibrary{arrows,positioning, shapes}
\usepackage{float}

\usepackage{setspace}
\singlespace

\def\R{\mathbb{R}}

\def\eps{\varepsilon}

\def\D{{\mathcal D}}

\def\G{{\mathcal G}}

\def\T{{\mathcal T}}
\def\X{{\mathcal X}}

\def\M{{\mathcal M}}

\def\S{{\mathcal S}}

\def\Z{{\mathcal Z}}
\def\U{{\mathcal U}}

\usepackage[font=small,labelsep=space]{caption}
\DeclareCaptionLabelSeparator{dot}{.~}
\newtheorem{definition}{Definition}
\newtheorem{theorem}{Theorem}
\newtheorem{corollary}{Corollary}

\newtheorem{proposition}{Proposition}

\theoremstyle{remark}
\newtheorem{remark}{Remark}

\usepgfplotslibrary{fillbetween}
\usepackage{times}
\usepackage{tikz}
\usepackage{amsmath}
\usepackage{verbatim}
\usetikzlibrary{arrows,shapes}
\tikzstyle{RectObject}=[rectangle,fill=white,draw,line width=0.2mm]
\tikzstyle{line}=[draw]
\tikzstyle{arrow}=[draw, -latex]
\usetikzlibrary{decorations.pathmorphing}
\usetikzlibrary{calc,shapes, positioning}
\usepackage{graphicx}
\usepackage{caption}

\newcommand{\paren}[1]{\left(#1\right)}
\newcommand{\bracket}[1]{\left[#1\right]}
\newcommand{\set}[1]{\left\{#1\right\}}
\newcommand{\norm}[1]{\left\Vert#1\right\Vert}
\newcommand{\abs}[1]{\left|#1\right|}
\newcommand{\parenBar}[2]{\paren{#1{\left\Vert\,#2\right.}}}

\definecolor{DukeBlue}{HTML}{001A57}
\definecolor{DarkRed}{rgb}{0.75, 0.0, 0.0}
\definecolor{DarkGreen}{rgb}{0.0, 0.5, 0.0}

\IEEEoverridecommandlockouts

\allowdisplaybreaks
\begin{document}

\title{\LARGE \bf
A Tamper-Free Semi-Universal Communication System for Deletion Channels
}

\author{Shahab Asoodeh$^{1}$, Yi Huang$^{2}$, and Ishanu Chattopadhyay$^{3}$
\thanks{$^{1}$Computation Institute and Institute of Genomics and System Biology, The University of Chicago, Chicago, IL 60637 
{\tt \small shahab@uchicago.edu}}
\thanks{$^{2}$ The University of Chicago, Chicago, IL
        {\tt\small yhuang10@uchicago.edu}}%
\thanks{$^{3}$Computation Institute, Chicago, IL
        {\tt\small ishanu@uchicago.edu}}
}

\maketitle

\begin{abstract}
	We investigate the problem of reliable communication between two legitimate parties over deletion channels under an active eavesdropping (aka jamming) adversarial model. To this goal, we develop a theoretical framework based on probabilistic finite-state automata to define novel encoding and decoding schemes that ensure small error probability in both message decoding as well as tamper detecting. We then experimentally verify the reliability and tamper-detection property of our scheme.
\end{abstract}
	
\section{Introduction}
The deletion channel is the simplest point-to-point communication channel that models synchronization errors. In the simplest form, the inputs are either deleted independently with probability $\delta$  or transmitted over the channel noiselessly. As a result, the length of channel output is a random variable depending on $\delta$. Surprisingly, the capacity of deletion channel has been one of the outstanding open problems in information theory \cite{mitzenmacher2009}. A random coding argument for proving a Shannon-like capacity result for deletion channel (in general for all channels with synchronization errors) was given by Dobrushin \cite{Dob67} which is recently improved by Kirsch and Drinea \cite{deletion2} to derive several lower bounds. Readers interested in most recent results on deletion channels are referred to the recent survey by Mitzenmacher \cite{deletion1} that provides a useful history and known results on deletion channels. 

As the problem of computing capacity of deletion channels is infamously hard, we focus on another problem in deletion channels. In this paper, we study the behavior of the deletion channel under an \textit{active} eavesdropper attack. Secrecy models in information theory literature, initiated by Yamamoto \cite{yamamotoequivocationdistortion},  assume that there exists a \emph{passive} eavesdropper who can observe the symbols being transmitted over the channel. The objective is to design a pair of (randomized) encoder and decoder such that the message  is decoded with asymptotically vanishing error probability at the legitimate receiver while ensuring  that the eavesdropper gains negligible information about the message. In all secrecy models (see, e.g., \cite{Secure_lossless_compression, Secure_lossy_coding, Cover_State_Amplification, secrecy1, secrecy2, secrecy3, Asoodeh_Allerton2015}) the crucial assumption is that the eavesdropper can neither jam the communication channel between legitimate parties nor can she modify any messages exchanged between them. 
However, in many practical scenarios, the eavesdropper can potentially change the channel, for instance, add stronger noise to change the crossover probability of a binary symmetric channel  or the deletion probability of a deletion channel.

In our adversarial model, we assume that two parties (say Alice and Bob) wish to communicate over a public deletion channel while an eavesdropper (say Eve)  can potentially tamper the statistics of the channel. We focus on deletion channel and assume that Eve can have possibly more bits deleted, and hence increases the deletion probability of the channel. The objective is to allow a reliable communication between Alice and Bob (with vanishing error probability) regardless of the eavesdropper's action. To this goal, we design (i) a randomized encoder using probabilistic finite-state automata which, given a fixed message, generates a random vector as the channel input and (ii) a decoder which generates an estimate of the message \emph{only} when the channel is not tampered. In case the channel is indeed tampered, the decoder can declare it with asymptotically  small Type I and Type II error probabilities. It is worth mentioning that the \emph{rate} of our scheme is (almost) zero and hence we do not intend to study capacity of deletion channels.  
	
Unlike the classical channel coding where the set of all possible channel inputs (aka, codebook) must be available at the decoder, our scheme requires that only the set of PFSA's used in the encoder to be available at the decoder. This model, that we call \emph{semi-universal}, is contrasted with \emph{universal} channel coding \cite{universal1} where neither channel statistics nor codebook are known and the decoder is required to find the \emph{pattern} of the message.

The rest of the paper is organized as follows. In Section~\ref{sec:PFSA}, we discuss briefly the notion of PFSA and its properties required for our scheme. Section~\ref{Sec:ChannelMOdel} specifies the channel model, encoder, decoder, and different error events.  In Section~\ref{Sec:PFSA_DELETION}, we discuss the effects of deletion channels on PFSA. Section~\ref{Sec:MLA} concerns the thoeretical aspects of our coding scheme and Section~\ref{Alg&Simulation} contains several experimental results.

\noindent\textit{\underline{Notation}}
We use calligraphic uppercase letters for sets (e.g.~$\S$), sans serif font for functions (e.g.~$\mathsf{T}$), uppercase letters for matrices (e.g.~$\Gamma$), bold lower case letters for vectors (e.g.~$\mathbf{v}$). Throughout, we use $g$ to denote a PFSA and $s$ and $x$ to denote its state and symbol, respectively. We use $\mathbf{x}^n=x_1\dots x_n$ for a sequence of symbols or interchangeably, $\mathbf{x}$ if its size is clear in context. Also, $\mathbf{v}_{i}$ for $i$th entry of vector $\mathbf{v}$, $A_{i, \cdot}$ and $A_{\cdot, j}$ for the $i$th row or column of the matrix $A$, respectively. We use $(a_x)_{x\in\X}$ to denote a vector with the entry indexed by $x$ and  $(\mathbf{a}_x)_{x\in\X}$ a matrix with the column indexed by $x$. Finally, $\mathbf{x}^{i}=x_1x_2\dots x_i$.


\section{Probabilistic finite state automata}\label{sec:PFSA}
In this section, we introduce a new measure of similarity between two vectors. To do this, we first need to define probabilistic finite-state automata (PFSA). 
\begin{definition}[PFSA]
	A probabilistic finite-state automaton is a quadruple $(\mathcal{S}, \mathcal{X}, \mathsf{T}, \mathsf{P})$, where $\mathcal{S}$ is a finite state space, $\mathcal{X}$ is a finite alphabet with 	$K=|\mathcal{X}|$, $\mathsf{T}:\mathcal{S}\times \mathcal{X}\to \mathcal{S}$ is the state transition function, and $\mathsf{P}:\mathcal{S}\times \mathcal{X}\to [0,1]$ specifies the conditional distribution of generating a symbol conditioned on the state.    
\end{definition}
In fact, a PFSA is a directed graph with a finite number of vertices (i.e., states) and directed edges emanating from each vertex to the other. An edge from state $s_1\in \mathcal{S}$ to state $s_2\in \mathcal{S}$ is specified by two labels: (i) a symbol $x\in \mathcal{X}$ that updates the current state from $s_1$ to $s_2$, that is, $\mathsf{T}(s_1, x)=s_2$, and (ii) the probability of generating $x$ when the system resides in state $s_1$, i.e., $\mathsf{P}(s_1, x)$. For instance, $\mathsf{P}(s_1, 1)=0.7$ in the PFSA described in Fig. \ref{fig:PFSA_M4},  thus, the system residing in states $s_1$ evolve to state $s_2$ with probability $0.7$ and it generates symbol $1$. Clearly, $\sum_{x\in \X}\mathsf{P}(s, x)=1$ for all $s\in \S$. 
\begin{figure}[t]
	\centering
    \includegraphics[scale=.8]{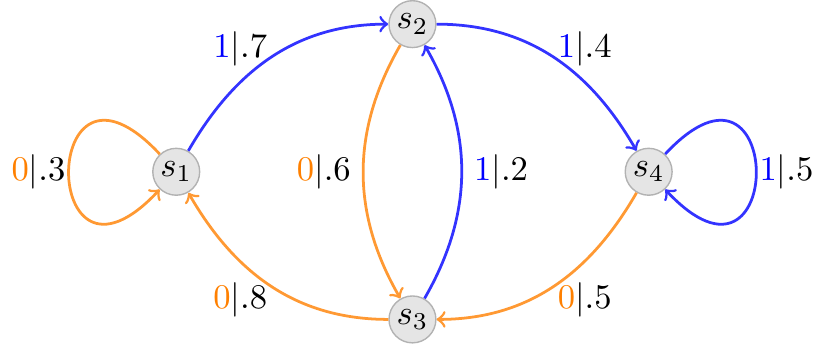}
	\caption{A PFSA with $\mathcal{S}=\{s_1, s_2, s_3, s_4\}$ and $\mathcal{X}=\{0, 1\}$.}
    \label{fig:PFSA_M4}
\end{figure}
Given two symbols $x_1$ and $x_2$, one can define the transition function for the concatenation $x_1x_2$ as $\mathsf{T}(s, x_1x_2)=\mathsf{T}(\mathsf{T}(s, x_1), x_2)$. Letting $\X^*$ denote the set of all possible concatenation of finitely many symbols from $\X$, one can easily proceed to define  $\mathsf{T}(s, \mathbf{x})$ as above for each  $\mathbf{x}\in \X^*$ and $s\in \S$. We say that a PFSA is \emph{strongly connected} if for any pair of distinct states $s_i$ and $s_j$, there exists a sequence $\mathbf{x}\in \X^*$ such that $\mathsf{T}(s_i, \mathbf{x})=s_j$. Let $\G$ be the set of all strongly connected PFSAs. The significance of strongly connected PFSAs is that their corresponding Markov chains (i.e., the Markov chain with state space $\S$ and transition matrix $P=\left[P(i, j)\right]_{|\S|\times |\S|}$ whose entry is $P(i, j)=\sum_{x\in \X: \mathsf{T}(s_i, x)=s_j}\mathsf{P}(s_i, x)$) has a unique stationary distribution (thus initial state can be assumed to be irrelevant).
\begin{definition}[$\Gamma$-expression for PFSA]
We notice that a PFSA $g$ is uniquely determined by $\mathbf{\Gamma}_{g}=(\Gamma_{g,x})_{x\in\X}$ given by 
\[
    \paren{\Gamma_{g,x}}_{i,j} = \left\{
    \begin{array}{ll}
        \mathsf{P}_g(s_i, x), & \qquad\mathsf{T}_g(s_i, x) = s_j,\\
        0, & \qquad\textrm{otherwise}.
    \end{array}
    \right.
\]
The \textit{state-to-state transition matrix} $P_g$ is defined as
\begin{equation}\label{Eq:State-StateTra}
    P_g=\sum_{x\in\X}\Gamma_{g,x},
\end{equation}

and the \textit{state-to-symbol transition matrix} $\widetilde{P}_g$ is given by
\[
    \widetilde{P}_g = (\Gamma_{g,x}\mathbf{1}_{|\S|})_{x\in\X},
\]
where $\mathbf{1}_n$ is the length-$n$ all-one vector. 
\end{definition}
For the PFSA illustrated in Fig.~\ref{fig:PFSA_M4}, we have
\[
    \Gamma_{g, 0} = \begin{pmatrix}
        .3 & 0 & 0 & 0\\
        0 & 0 & .6 & 0\\
        .8 & 0 & 0 & 0\\
        0 & 0 & .5 & 0
    \end{pmatrix}, \quad \Gamma_{g, 1} = \begin{pmatrix}
        0 & .7 & 0 & 0\\
        0 & 0 & 0 & .4\\
        0 & .2 & 0 & 0\\
        0 & 0 & 0 & .5
    \end{pmatrix},
\]
\[
    P_g = \begin{pmatrix}
        .3 & .7 & 0 & 0\\
        0 & 0 & .6 & .4\\
        .8 & .2 & 0 & 0\\
        0 & 0 & .5 & .5
    \end{pmatrix}, ~~~\text{and}~~~\quad
    \widetilde{P}_g = \begin{pmatrix}
        .3 & .7\\
        .6 & .4\\
        .8 & .2\\
        .5 & .5
    \end{pmatrix}.
\]

\begin{definition}[Generalized PFSA]
Generalized PFSA is a PFSA $g$ whose $\Gamma_{g, x}$ can have more than one non-zero (positive) entries.  In this case, we still have
\[
    \paren{\Gamma_{g,x}\mathbf{1}_{|\S|}}_{i} = \mathsf{P}_g(s_i, x).
\]
 However, $\mathsf{T}(s_i, x)$ might not be deterministic, and instead it is a probability distribution.
\end{definition}

Shannon \cite{Shannon} appears to be the first one who made use of PFSAs to describe stationary and ergodic sources. Given $g\in \G$, first a state $s_1$ is chosen randomly according to the stationary distribution, then a symbol $x_1$ is generated with probability $\mathsf{P}(s_1,x_1)$ which takes the system from state $s_1$ to state $s_2$. A new symbol $x_2$ is then generated with probability $\mathsf{P}(s_2,x_2)$. Letting this process run for $n$ time units, we obtain a sequence $x_1, x_2, \dots, x_n$. In this case, we say that $x_1, x_2, \dots, x_n$ is a \textit{realization} of $g$. According to Shannon, each state $s_i$ captures the "residue of influence" of the preceding symbol $x_{i-1}$ on the system.

For $\mathbf{x}\in \X^*$, we denote by $\mathbf{x}\leftarrow g$ the fact that $g\in \G$ generates $\mathbf{x}$. 
\begin{figure}[h!]
	\centering
	\includegraphics[scale = 1.]{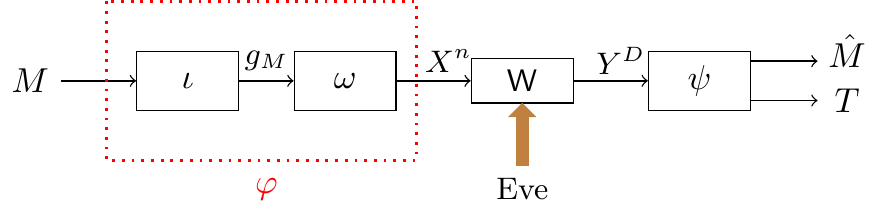}
	\caption{A communication system with an active eavesdropper}
	\label{Fig1}
\end{figure}

\section{System Model and Setup}\label{Sec:ChannelMOdel}
Suppose Alice has a message $M$ which takes value in a finite set $\M\coloneqq \{1, 2, \dots, |\M| \}$ and seeks to transmit it reliably to Bob over a  deletion channel  $\mathsf{W}(\delta)$ with deletion probability $\delta\in [0,1]$. The communication channel is assumed to be  public, that is, an active eavesdropper, say Eve, can access and possibly tamper the channel. For simplicity, we assume that Eve may delete extra bits and thus changing the channel from $\mathsf{W}(\delta)$ to $\mathsf{W}(\delta')$ with $\delta'\geq \delta$.

\newcommand{\ccirc}{\mathbin{\mathchoice
		{\xcirc\scriptstyle}
		{\xcirc\scriptstyle}
		{\xcirc\scriptscriptstyle}
		{\xcirc\scriptscriptstyle}
}}
\newcommand{\xcirc}[1]{\vcenter{\hbox{$#1\circ$}}}

The objective is to design a pair of encoder $\varphi$ and decoder $\psi$ that enables Alice and Bob to reliably communicate over $\mathsf{W}(\delta)$ \emph{only} when he is ensured that the channel is not tampered. In classical information theory, the decoder must be tuned with the channel statistics. Hence, reliable communication occurs only when Bob knows the deletion probability $\delta$. However, Eve might have tampered the channel and increased deletion probability to $\delta'$, and since Bob's decoding policy was tuned to $\delta$, this might cause a decoding error --regardless of Bob's decoding algorithm. Therefore, reliability of the decoding must be always conditioned on the fact that the channel has not been tampered during communication. 

Motivated by this observation, we propose the following coding scheme.  We first propose a two-step encoder: each  message $M=m$ is first sent to a function $\iota:\M\to \G_\M$ which maps  $m$ to a PFSA $g_m$  in $\G_\M\coloneqq\{g_1, \dots, g_{|\M|}\}$, then  another function $\omega:\G_\M\to \X^n$ generates $\mathbf{y}^n$ a realization of PFSA $g_m$ and sends it over the memoryless channel $\mathsf{W}(\delta)$. Therefore, the encoder function $\varphi:\M\to \X^n$ is the composition $\iota\ccirc \omega$ (see Fig. \ref{Fig1}). Unlike the classical setting, Bob need not know the set of all channel inputs $\mathbf{y}^n$ for each $m\in \M$ (aka codebook). Instead, we assume Bob knows $\G_\M$ (thus the name \emph{semi-universal scheme}). The output of the channel $\mathbf{x}^D$ is an $\X$-valued random vector whose length $D$ is a binomial random variable $\mathsf{Bin}(n, 1-\delta)$ (corresponding to how many elements of $\mathbf{y}^n$ are deleted).  Upon receiving $\mathbf{x}^D$, Bob  applies $\psi:\X^*\to \M\times \{0,1\}$ to generate $\psi(\mathbf{x}^D)=(\hat{M}, T)$ where $\hat{M}$ is an estimate of Alice's message  and $T$ specifies whether or not the channel has been tampered. He then declares $\hat{M}$ as the message only when $T=0$. Therefore, the goal is to design $(\varphi, \psi)$ such that for sufficiently large $n$
\begin{equation}
\Pr(T=0 ~|~ \text{channel~is~tampered})+\Pr(T=1 ~|~ \text{channel~is~not~tampered})<\eps \label{Tamper_Prob},
\end{equation}
and simultaneously
\begin{equation}\label{Error_Prob}
\Pr(M\neq \hat{M}|T=0)\leq \eps,
\end{equation}
for any uniformly chosen message $M\in \M$. 
We say that the reliable tamper-free communication is possible if \eqref{Tamper_Prob} and  \eqref{Error_Prob} hold simultaneously for any $\eps>0$.

\section{PFSA through deletion channel}\label{Sec:PFSA_DELETION}
In this section, we study the channel effect on PFSA's by monitoring the change of the likelihood of $\mathbf{x}$ being generated by a PFSA at the channel output. To do this, we first study the likelihood when $\delta = 0$ in Section \ref{subsec:PFSA_without_Deletion}, and then move on to the case of positive $\delta$ in Section \ref{subsec:PFSA_with_Deletion}. One of the main results in this section is to show that the output of $\mathsf{W}(\delta)$ (i.e., $\mathbf{x}^D$) can be equivalently generated by a \emph{generalized} PFSA $g(\delta)$ whose $\Gamma$ and state-to-state transition matrix follow simple closed forms (cf. Theorem \ref{thm:generalizedPFSA}). In section \ref{subsec:propertyGeneralizedPFSA}, we discuss some basic properties of $g(\delta)$ that will be useful for later development. We conclude this section by introducing the class M2 of PFSAs which is closed under deletion. For notational brevity, we remove the subscript $g$ when it is is clearly understood from context. 

\subsection{PFSA over $\mathsf{W}(0)$: no deletion}\label{subsec:PFSA_without_Deletion}
Let a sequence of symbols $\mathbf{x} = x_1\dots x_n \in\X^{*}$ be given. We define   $p_g(\mathbf{x})$ (or simply $p(\mathbf{x})$) to be the probability that $g$ generates $\mathbf{x}$. Then we have 
\[
    p(\mathbf{x}^{n}) = p(x_1)p\paren{x_2|\mathbf{x}^1}\cdots p\paren{x_n|\mathbf{x}^{n-1}},
\]
where $p\paren{x_i|\mathbf{x}^{i-1}}$ is the conditional probability of $g$ generating $x_i$ given that $g$ generated $\mathbf{x}^{i-1}$. It is clear from section~\ref{sec:PFSA} that 
\begin{align}\label{Def_Likelihood}
    &\mathbf{p}_0 = \mathbf{p}\\
    p(x_1) = \paren{\mathbf{p}_0^{T}\widetilde{P}}_{x_1}, &~\mathbf{p}^T_1 = \frac{\mathbf{p}_0^{T}\Gamma_{x_1}}{\norm{\mathbf{p}_0^{T}\Gamma_{x_1}}_1},\nonumber\\
    p\paren{x_2|x_1} = \paren{\mathbf{p}_1^{T}\widetilde{P}}_{x_2}, & ~\mathbf{p}^T_2 = \frac{\mathbf{p}^{T}_1\Gamma_{x_2}}{\norm{\mathbf{p}^{T}_1\Gamma_{x_2}}_1},\nonumber\\
    & \vdots\nonumber\\
    p\paren{x_{n-1}|\mathbf{x}^{n-2}} = \paren{\mathbf{p}_{n-2}^{T}\widetilde{P}}_{x_{n-1}}, & ~\mathbf{p}^T_{n-1} = \frac{\mathbf{p}_{n-2}^{T}\Gamma_{x_{n-1}}}{\norm{\mathbf{p}_{n-2}^{T}\Gamma_{x_{n-1}}}_1},\nonumber
\end{align}
and finally, $p\paren{x_n|\mathbf{x}^{n-1}} =\paren{\mathbf{p}_{n-1}^{T}\widetilde{P}}_{x_n}$, where $T$ denotes matrix transpose.


It is clear from the above update rule that any sequence $\mathsf{x^n}$ induces two  probability distribution: one on the state space $\S$, i.e., $\mathbf{p}_{n}$ and the other one on $\X$. Let denote the former by $\mathbf{p}_g(\mathbf{x})$ and the latter by $\D_g(\mathbf{x})$.  Update rules in \eqref{Def_Likelihood} imply that $\D_g(\mathbf{x})= \mathbf{p}^T_g(\mathbf{x})\widetilde{P}_g$ and $\mathbf{p}^{T}_g(\mathbf{x}x) \propto \mathbf{p}^{T}_g(\mathbf{x})\Gamma_{g,x}$. More precisely, since
\begin{equation*}
    \norm{\mathbf{p}^{T}_g(\mathbf{x})\Gamma_{g,x}}_1 = \mathbf{p}^{T}_g(\mathbf{x})\Gamma_{g,x}\mathbf{1}_{|\S|} = \mathbf{p}^{T}_g(\mathbf{x})\paren{\widetilde{P}_g}_{\cdot, x}=\paren{\mathbf{p}^{T}_g(\mathbf{x})\widetilde{P}_g}_{x} = p(x|\mathbf{x}),
\end{equation*}
we have 
\begin{equation}\label{eq:inductionDistrOfState}
    p^T(x|\mathbf{x})\mathbf{p}_g(\mathbf{x}x) = \mathbf{p}^T_g(\mathbf{x})\Gamma_{g,x}.
\end{equation}
We also call $\D_g(\mathbf{x}) = \paren{p_g(x|\mathbf{x})}_{x\in\X}$ the \emph{symbolic derivative} of $g$ induced by $\mathbf{x}$.

\subsection{PFSA over $W(\delta)$: deletion with probability $\delta>0$}\label{subsec:PFSA_with_Deletion}
Now we move forward to investigate the effect of deletion probability on PFSA transmission. The following result is a ket for our analysis.

\begin{theorem}\label{thm:generalizedPFSA}
Let $\mathbf{y}\leftarrow g$ be a channel input and $\mathbf{x}$ be a channel output with positive deletion probability $\delta$. Then $\mathbf{x}\leftarrow g(\delta)$, where $g(\delta)$ is a generalized PFSA identified by $\Gamma_{g, x, \delta} = Q(P, \delta)\Gamma_{g, x}$ for all $x\in\X$, where $P$ is the state-to-state transition matrix of $g$ and $Q$ is as defined in \eqref{eq:Q}.
\end{theorem}
\begin{proof}
Assume Bob has observed $\mathbf{x}^{i-1}$. Then we have
\begin{eqnarray*}
    p(x_i|\mathbf{x}^{i-1}) &=& (1-\delta)\paren{\mathbf{p}^T_{i-1}\widetilde{P}}_{x_i}+\delta(1-\delta)\paren{\mathbf{p}^T_{i-1}P\widetilde{P}}_{x_i}+\delta^2(1-\delta)\paren{\mathbf{p}^T_{i-1}P^2\widetilde{P}}_{x_i}+\cdots\\
    &=&(1-\delta)\paren{\mathbf{p}^T_{i-1}\paren{\sum_{i = 0}^{\infty}\delta^{i}P^{i}}\widetilde{P}}_{x_i}\\
    &=& \paren{\mathbf{p}^T_{i-1}Q(P,\delta)\widetilde{P}}_{x_i},
\end{eqnarray*}
where 
\begin{equation}\label{eq:Q}
    Q(P, \delta) = (1-\delta)\sum_{i = 0}^{\infty}\delta^iP^{i} = (1-\delta)\paren{I-\delta P}^{-1}.
\end{equation}
Analogous to \eqref{Def_Likelihood}, we can define the follwoing distribution induced on $\S$
\begin{equation}\label{Distri_Deletion}
    \mathbf{p}_{i} = \frac{\mathbf{p}^T_{i-1}Q(P,\delta)\Gamma_{x_i}}{\norm{\mathbf{p}^T_{i-1}Q(P,\delta)\Gamma_{x_i}}_1}.
\end{equation}
Comparing \eqref{Distri_Deletion} with expressions $\mathbf{p}_i$ in \eqref{Def_Likelihood}, the result follows. 
\end{proof}

\begin{remark}
Notice that while the row-stochastic matrix $P$ may not be invertible, $I-\delta P$ is non-singular for all $\delta\in [0,1)$, as the the eigenvalues of $P$ are less than or equal to $1$. Moreover, it is clear from \eqref{eq:Q} that $Q(P, \delta)$  is also a row-stochastic matrix with $\mathbf{p}$ being its eigenvector corresponding to eigenvalue one. We will give a closer look at the eigenvalues of $Q(P, \delta)$ in the next section.
\end{remark}

\subsection{Properties of the generalized PFSA}\label{subsec:propertyGeneralizedPFSA}
We start by analyzing the eigenspace of the state-to-state transition matrix of $g(\delta)$. Note that it follows from \eqref{Eq:State-StateTra} that $P_{g(\delta)} = Q\paren{P_g, \delta}P_g$. 

\begin{theorem}\label{Thm:DEltaMachine}
Let $\mathbf{p}_g$ be the stationary distribution of strongly connected  $g$. Then the generalized PFSA $g(\delta)$ is also strongly connected with stationary distribution $\mathbf{p}_{g(\delta)} = \mathbf{p}_g$.
\end{theorem}
\begin{proof}
Let $\lambda$ be an eigenvalue of $P_g$. Then  $\lambda(1-\delta)(1-\delta\lambda)^{-1}$ is an eigenvalue of $P_{g(\delta)}$. Define $f(\lambda, \delta) = \lambda(1-\delta)(1-\delta\lambda)^{-1}$. Then the result follows from the following observations:
\begin{enumerate}
    \item For $\lambda = 1$, $f(\lambda, \delta) = 1$ for all $\delta \in [0, 1)$, and hence $\lim_{\delta\rightarrow1}f(1, \delta) = 1$.
    \item  For $\lambda < 1$, $f(\lambda, \delta) < \lambda$ for all $\delta \in [0, 1)$, and furthermore, $\lim_{\delta\rightarrow1}f(\lambda, \delta) = 0$. \qedhere
\end{enumerate}
\end{proof}

Then following is an immediate corollary. 
\begin{corollary}
We have for all $x\in \X$
$$p_g(x)=p_{g(\delta)}(x).$$
\end{corollary}
\begin{proof}
We have
\[
    \mathbf{p}_{g(\delta)}^{T}\widetilde{P}_{g(\delta)} = \mathbf{p}_g^{T}\widetilde{P}_{g(\delta)} = \mathbf{p}_g^{T}Q(P_g, \delta)\widetilde{P}_{g} = \mathbf{p}_g^{T}\widetilde{P}_{g}.\qedhere
\]
\end{proof}

A natural question is what happens when $\delta\uparrow 1$. Letting $g(1)$ denote the machine corresponding to $\delta\uparrow 1$, we now show that, quite expectedly, $g(1)$ is a single-state machine.

\begin{theorem}
$g(1)$ is a single-state PFSA.
\end{theorem}
\begin{proof}
First note that the observations given in the proof of Theorem~\ref{Thm:DEltaMachine} imply that 
\[
    \lim_{\delta\rightarrow1}Q(P_g, \delta) = \mathbf{1}_{|\S|}\mathbf{p}_g^{T},
\]
and consequently $g(1)$  is a PFSA specified by $\mathbf{1}_{|\S|}\mathbf{p}_g^{T}\Gamma_{g, x}$ for $x\in\X$. 

Suppose $\mathbf{x}=x_1x_2\dots x_n$ is observed. Following the argument given in section \ref{subsec:PFSA_with_Deletion}, we get
\begin{eqnarray*}
    &&p_{g(1)}(\mathbf{x}x)\\
    &= &\mathbf{p}^{T}(\mathbf{1}\mathbf{p}^{T}\Gamma_{x_1})(\mathbf{1}\mathbf{p}^{T}\Gamma_{x_2})\cdots(\mathbf{1}\mathbf{p}^{T}\Gamma_{x_n})\paren{\widetilde{P}_{g(1)}}_{\cdot, x}\\
    &= &\mathbf{p}^{T}(\mathbf{1}\mathbf{p}^{T}\Gamma_{x_1})(\mathbf{1}\mathbf{p}^{T}\Gamma_{x_2})\cdots(\mathbf{1}\mathbf{p}^{T}\Gamma_{x_n})\paren{\mathbf{1}\mathbf{p}^{T}\Gamma_x\mathbf{1}}\\
    &=&\paren{\mathbf{p}^{T}\mathbf{1}}\paren{\mathbf{p}^{T}\Gamma_{x_1}\mathbf{1}}\cdots\paren{\mathbf{p}^{T}\Gamma_{x_n}\mathbf{1}}\paren{\mathbf{p}^{T}\Gamma_x\mathbf{1}},
\end{eqnarray*}
and hence, by induction, $p_{g(1)}(x|\mathbf{x}) = \mathbf{p}^{T}\Gamma_x\mathbf{1}$ for all $\mathbf{x}$.
Since an i.i.d.\ process corresponds to a single-state PFSA, we conclude that $g(1)$ is in fact a single-state PFSA. 
\end{proof}
\begin{figure}[t]
    \centering
    \subfigure[$\delta = 0$]{\includegraphics[scale=.85]{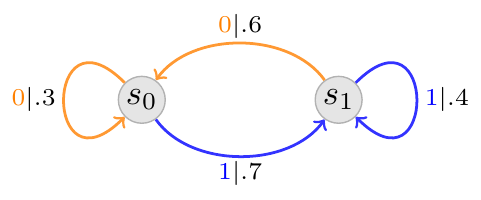}}
    \subfigure[$\delta = .25$]{\includegraphics[scale=.85]{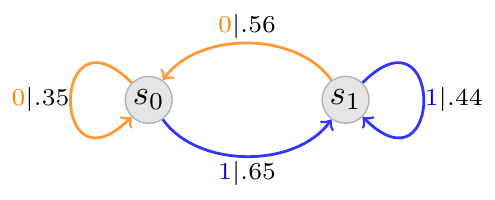}}
    \caption{On the left: $g_{(.3, .6)}$ in class M2. On the right, $g_{(.3, .6)}(.25)$, with transition probabilities rounded to two decimal places. We can see that deletion only cause the transition probabilities to change, but keep the structure of the machine.}
    \label{fig:PFSA_M2}
\end{figure}
\subsection{M2 Class of PFSA}
We note that $g(\delta)$ of a PFSA $g$ is not necessarily a PFSA. As an example, the $\Gamma$-expression of the generalized PFSA $g(.4)$ for $g$ being the PFSA described in Fig.~\ref{fig:PFSA_M4} is
\begin{figure}[h!]
    \centering
    \includegraphics[scale=0.9]{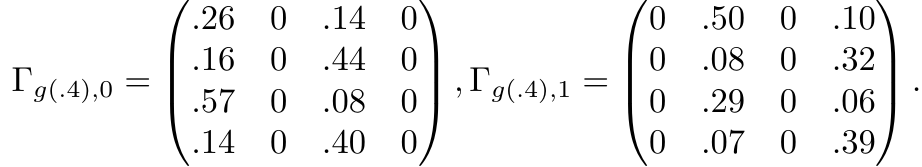}
\end{figure}

Nevertheless, we introduce M2 a class of PFSAs which is closed under deletion, i.e.\ $g\in \text{M2}$ implies  $g(\delta)\in \text{M2}$ for all $\delta\in[0, 1]$. As this class is instrumental in our experimental results, we shall study it in more details. 

M2 is the collection of $2$-state PFSAs on a binary alphabet: $g = g_{(\mu, \nu)}\in \text{M2}$ with $\mu, \nu \in (0, 1) \times (0, 1)$ is specified by a quadruple $\paren{\mathcal{S}, \mathcal{X}, \mathsf{T}, \mathsf{P}_{(\mu, \nu)}}$, where $\S = \set{s_0, s_1} \quad \X = \set{0, 1},$ and 

\[
\Gamma_{g_{(\mu, \nu)}, 0} = \begin{pmatrix}\mu & 0 \\ \nu & 0\end{pmatrix}, \quad \Gamma_{g_{(\mu, \nu)}, 1} = \begin{pmatrix}0 & 1-\mu \\0 & 1 - \nu\end{pmatrix}.
\]

Fig.~\ref{fig:PFSA_M2} illustrates $g_{(.3, .6)}$ and its corresponding $g_{(.3, .6)}(\delta)$, which is obtained from Theorem~\ref{thm:generalizedPFSA}. Since $\Gamma_{g, x, \delta}$ has exactly the same form -- containing a single column of non-zero entries for all $\delta$, it is clear that $g_{(.3, .6)}(\delta)\in \text{M2}$. 

Since each $g_{(\nu, \mu)}$ is specified by two numbers, we can parametrize M2 by a square in $\R^2$. In Fig.~\ref{fig:parametrization_M2}, we show the effect of deletion probability on M2 machines. The key observation is that deletion probability drives machines to $\mu=\nu$ line.  
\begin{figure}[ht]
    \centering
    \subfigure[$\delta = 0$]{\includegraphics[scale=.25]{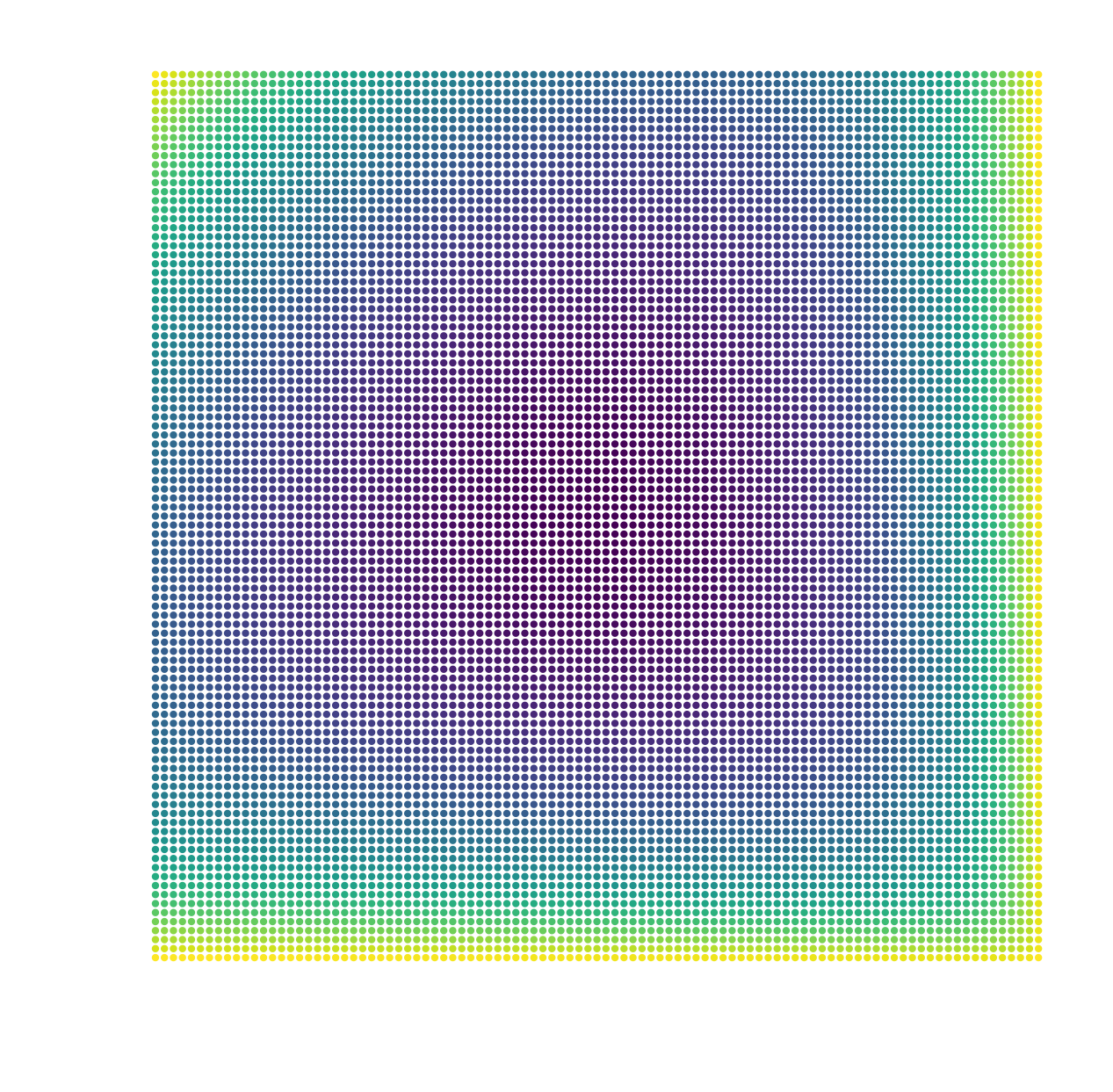}}
    \subfigure[$\delta = 0.25$]{\includegraphics[scale=.25]{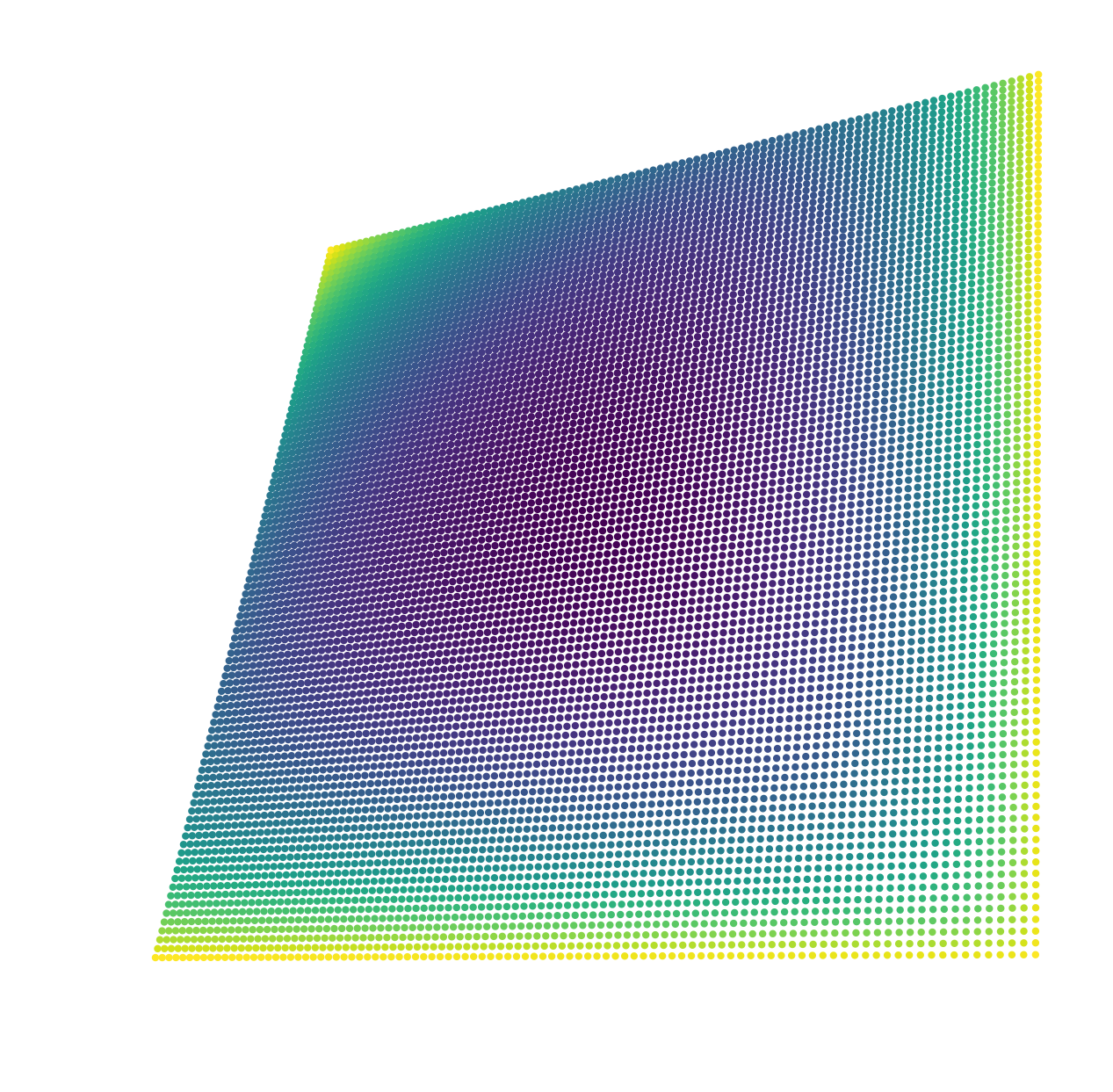}}
    \subfigure[$\delta = 0.50$]{\includegraphics[scale=.25]{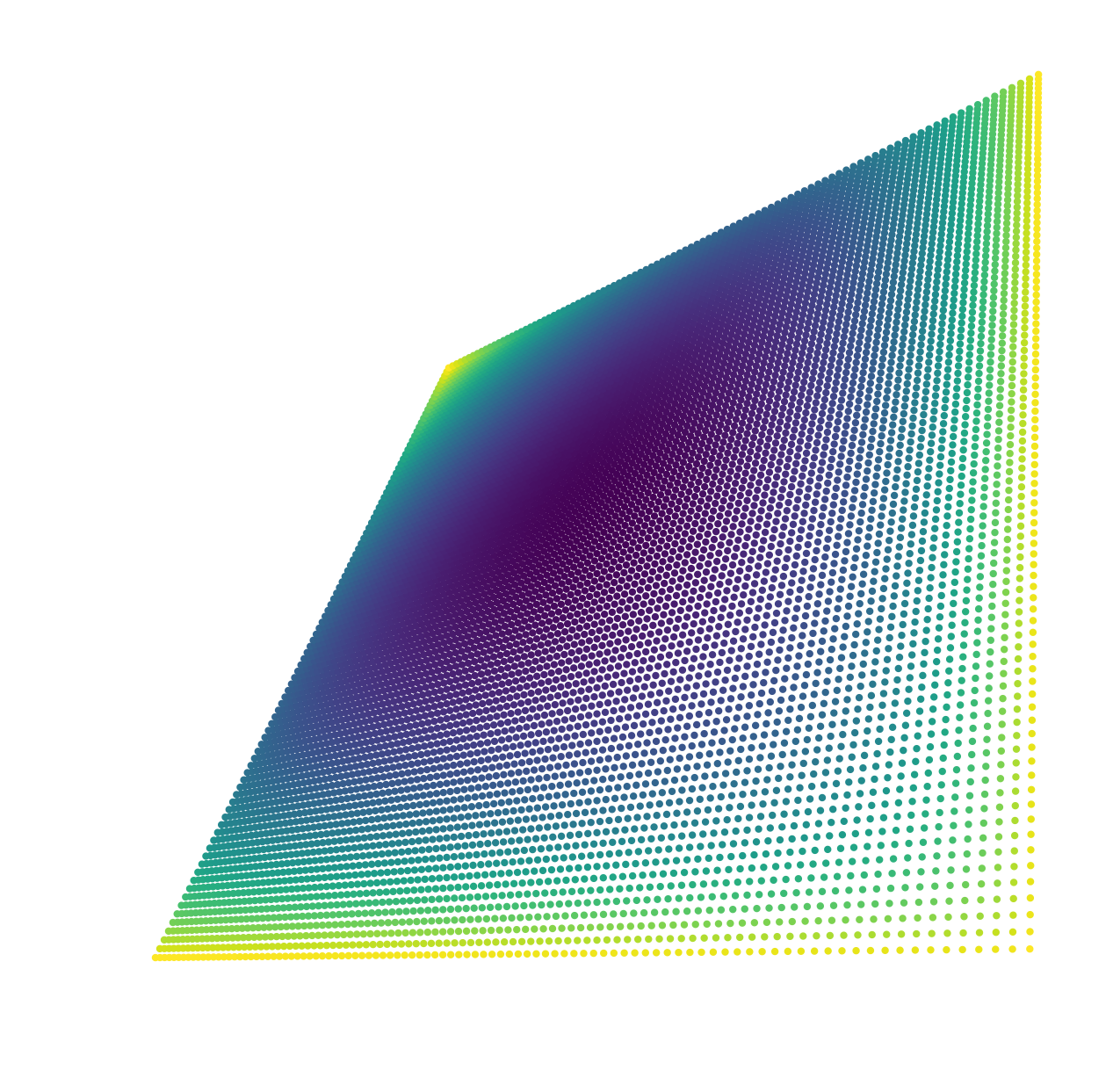}}
    \subfigure[$\delta = 0.75$]{\includegraphics[scale=.25]{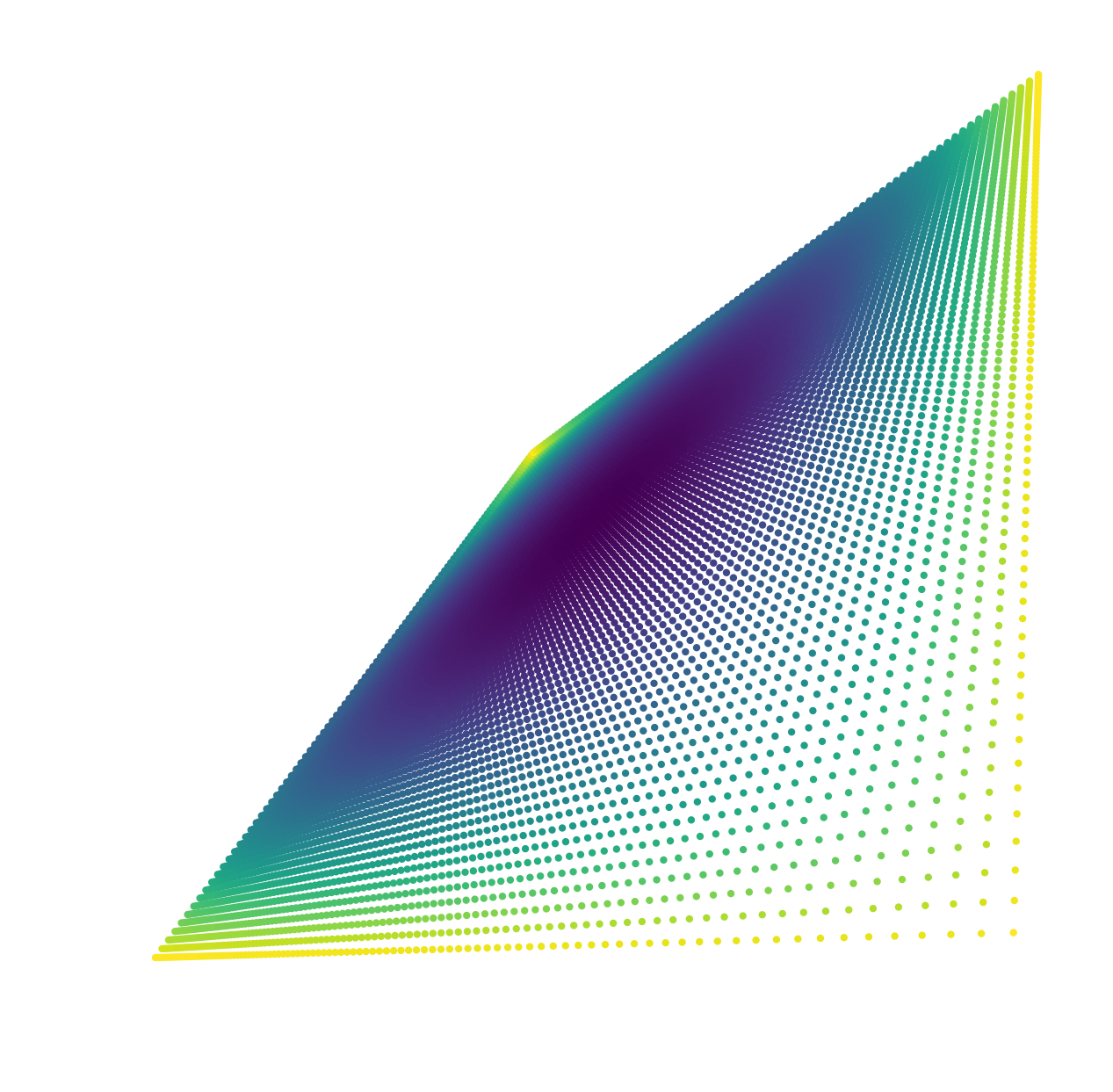}}
    \caption{Each dot in (a) represents a $g_{(\mu, \nu)}$ in M2 with $\mu, \nu$ both ranging from $0.01$ to $0.99$ and with $0.01$ increment. The color of the points is proportional to the KL divergence (defined in Section~\ref{subsec:KLDivergence}) of $g_{(.5, .5)}$ to $g$. The reason that the images are symmetric with respect to the $\mu + \nu = 1$ line is that $g_{(1-\nu, 1-\mu)}$ is exactly $g_{(\mu,\nu)}$ with the two states swapped. We can see that while we increase $\delta$, the dots are moving towards the $\mu = \nu$ line which corresponds to the single-state PFSA. The asymmetry in how fast PFSA on each side of the $\mu+\nu=1$ line converges to single-state PFSA is caused by structural difference between them -- machines on the upper side, with $\mu < \nu$,  have strong connections between two states, while machines on the lower side, with $\mu > \nu$, have weaker connection between the states.}
    \label{fig:parametrization_M2}
\end{figure}

\section{The convergence of likelihood}\label{Sec:MLA}
The goal of this section is to lay the theoretical ground for our algorithms for decoding and tamper detecting with PFSAs. In Section~\ref{subsec:ThmLlhConvergence}, we employ maximum likelihood framework to decode the generating PFSA given the channel output. We show that likelihood is closely related to entropy rate and KL divergence of PFSAs (to be defined and calculated in \ref{subsec:EntropyRate} and \ref{subsec:KLDivergence}). 

\subsection{Entropy rate of PFSA}\label{subsec:EntropyRate}
Let $g$ be a PFSA. We define $H_n(g)$ as the following: 
\[
    H_n(g) \coloneqq -\sum_{|\mathbf{x}| = n}p_g(\mathbf{x})\log{p_g(\mathbf{x})}.
\]
Then the entropy rate of $g$ is defined as 
$$H(g)\coloneqq\lim_{n\rightarrow\infty}\frac{1}{n}H_n(g).$$
Note that $H(g)$ is in fact the entropy rate of the stochastic process corresponding to $g$ \cite{Cover:EIT}.
In the next theorem, we show that the above limit exists and and the entropy rate has a simple closed form.
\begin{theorem}\label{thm:EntropyRate}
We have
\[
    H(g) = \sum_{s\in\S} (\mathbf{p}_g)_sH\paren{\paren{\widetilde{P}_g}_{s, \cdot}}
\]
\end{theorem}
\begin{proof}
See Appendix~\ref{Appendix_Entropy}.
\end{proof}
It readily follows from the theorem above that the entropy rate for $g_{(\mu, \nu)}$ is 
\begin{equation*}
    H\paren{g_{(\mu, \nu)}} = \frac{\nu h_b(\mu)}{\bar{\mu} + \nu} +  \frac{\bar{\mu}h_b(\nu)}{\bar{\mu} + \nu},
\end{equation*}
where $\bar{a} \coloneqq 1-a$ and $h_b(a)\coloneqq -a\log{a} - \bar{a}\log\bar{a}$ is the binary entropy function for any $a\in [0,1]$.

Next, we show that deletion increases entropy rate, which will be critical for tamper detection purpose.
\begin{theorem}\label{thm:MonotonicityDeltaM2}
The map $\delta\mapsto H(g_{(\mu, \nu)}(\delta))$ is monotonically increasing when $\mu\neq \nu$. 
\end{theorem}
\begin{proof}
We have
\[
    \mu(\delta) = \frac{\mu - \delta(\mu - \nu)}{1-\delta(\mu - \nu)}, \quad \nu(\delta) = \frac{\nu}{1-\delta(\mu - \nu)},
\]
and 
\begin{equation*}
    H\paren{g_{(\mu, \nu)}(\delta)} = \frac{\nu}{1- \mu + \nu}h_b\paren{\frac{\mu - \delta(\mu - \nu)}{1-\delta(\mu - \nu)}} + \frac{1-\mu}{1- \mu + \nu}h_b\paren{\frac{\nu}{1-\delta(\mu - \nu)}}.
\end{equation*}
We can then write
$$\frac{\text{d}}{\text{d}\delta}H(g_{(\mu, \nu)}(\delta))=\frac{\alpha\bar{\mu}\nu}{(1-\alpha\delta)^2\bar{\alpha}}\log{\frac{(\mu-\delta\alpha)(\bar{\nu}-\delta\alpha)}{\bar{\mu}\nu}},$$
where $\alpha=\mu - \nu$. It's straightforward to check that the derivative is always positive when $\mu \neq \nu$. 
\end{proof}

\subsection{KL divergence of two PFSAs}\label{subsec:KLDivergence}
Let $g_1, g_2\in \text{M2}$. The $n$-th order KL divergence between $g_1$ and $g_2$ is the KL divergence on the space of length-$n$ sequences, i.e. 
\[
    D_n\parenBar{g_1}{g_2} = \sum_{|\mathbf{x}| = n}p_{g_1}(\mathbf{x})\log{\frac{p_{g_1}(\mathbf{x})}{p_{g_2}(\mathbf{x})}}.
\]
Analogous to entropy rate, we can define the KL divergence between $g_1$ and $g_1$ as 
\[
    D_{\textrm{KL}}\parenBar{g_1}{g_2}\coloneqq \lim_{n\rightarrow\infty}\frac{1}{n}D_n\parenBar{g_1}{g_2}.
\]
We show in Theorem \ref{thm:KLDivergence} below shows that the limit exists and also derived a closed form for the KL divergence between two  PFSAs. But before we can state the theorem, we need to introduce a very useful construction on two PFSAs, called \textit{synchronous composition}. 
\begin{definition}[synchronous composition]
Let $g_1=\paren{\S, \X, \mathsf{T}_1, \mathsf{P}_1}$ and $g_2=\paren{\T, \X, \mathsf{T}_2, \mathsf{P}_2}$ be two PFSAs with the same alphabet and let $g^{*}_{\textrm{c}}\parenBar{g_1}{g_2}$ be the probabilistic automata specified by the quadruple $\paren{\S_{\textrm{c}}, \X, \mathsf{T}_{\textrm{c}}, \mathsf{P}_{\textrm{c}}}$ where
\begin{align*}
    \S_{\textrm{c}} &= \S_1 \times \T = \set{(s, t)}_{s\in\S_1, t\in\T}
\end{align*}
is the Cartesian product of $\S$ and $\T$, and
\begin{align*}
    \mathsf{T}_{\textrm{c}}((s, t), x) &= \paren{\mathsf{T_1}\paren{s, x}, \mathsf{T_2}\paren{t, x}},\\
    \mathsf{P}_{\textrm{c}}((s, t), x) &= \mathsf{P_1}\paren{s, x},
\end{align*}
for all $s\in\S$, $t\in\S$, and $x\in\X$. Then the synchronous composition $g_{\textrm{c}}\parenBar{g_1}{g_2}$ is defined to be any absorbing strongly connected component of $g^{*}_{\textrm{c}}\parenBar{g_1}{g_2}$, i.e. strongly connected component without any out-going edges.
\end{definition}
It is not clear that there is only one absorbing strongly connected  component in $g^{*}_{\textrm{c}}\parenBar{g_1}{g_2}$. However, as proved in Theorem~\ref{Thm:Equivalence} in Appendix~\ref{Appendix:KL}, $g_{\textrm{c}}\parenBar{g_1}{g_2}$ is \textit{equivalent} to $g_1$ irrespective of the choice of absorbing strongly connected  component, i.e.,  $p_{g_{\textrm{c}}}\paren{\mathbf{x}} = p_{g_1}\paren{\mathbf{x}}$ for $\mathbf{x}\in\X^{*}$.

In Figs.~\ref{fig:SyncCompM2&M2}, \ref{fig:SyncCompM2&S}, \ref{fig:SyncCompM2&T}, and \ref{fig:SyncCompM4&M2}, we provide examples of synchronous compositions for several $g_1$ and $g_2$ which shed light on the fact that the synchronous composition of two strongly connected PFSA might not be strongly connected.
\begin{theorem}\label{thm:KLDivergence}
Let $g_{\textrm{c}}=g_{\textrm{c}}\parenBar{g_1}{g_2}$ and $\mathbf{p}_{g_{\textrm{c}}}$ be the stationary distribution of $g_{\textrm{c}}$. Then we have
\begin{align*}
    &\lim_{n\rightarrow\infty}\frac{1}{n}D_n\parenBar{p_{g_1}^{n}}{p_{g_1}^{n}} = \sum_{s\in\S, t\in\T}{\paren{\mathbf{p}_{g_{\textrm{c}}}}}_{(s,t)}D_{\textrm{KL}}\parenBar{\paren{\widetilde{P}_{g_1}}_{s,\cdot}}{\paren{\widetilde{P}_{g_2}}_{t,\cdot}}.
\end{align*}
\end{theorem}
\begin{proof}
See Appendix~\ref{Appendix:KL}.
\end{proof}
In light of this theorem,  one can easily show   
\begin{equation*}
    D_{\textrm{KL}}\parenBar{g_1}{g_2} = \frac{\nu_1 D_{\textrm{KL}}\parenBar{\mu_1}{\mu_2}}{\bar{\mu}_1 + \nu_1} + \frac{\bar{\mu}_1D_{\textrm{KL}}\parenBar{\nu_1}{\nu_2}}{\bar{\mu}_1 + \nu_1}.
\end{equation*}

\subsection{Convergence of log likelihood}\label{subsec:ThmLlhConvergence}
According to Shannon-McMillan-Breiman Theorem \cite[Theorem 16.8.1]{Cover:EIT}, we have $-\frac{1}{n}\log p_{g}(\mathbf{x})\to H(g)$ for any sequence $\mathbf{x}\leftarrow g$. A natural question is that what the log-likelihood converges to if $\mathbf{x}$ is generated by a different machine. The following theorem states that the log-likelihood converges to entropy of generating machine plus the KL divergence which accounts for the mismatch.
\begin{theorem}\label{thm:LLHConvergenceM2}
For any $\mathbf{x}^n\leftarrow g\in M2$, we have with probability one
\begin{equation*}
    -\frac{1}{n}\sum_{i = 1}^{n}\log{p_{g'}\paren{x_i|\mathbf{x}^{i-1}}}\rightarrow H(g) + D_{\textrm{KL}}\parenBar{g}{g'},
\end{equation*}
for any PFSA $g'\in M2$.
\end{theorem}
\begin{proof}
First note that 
\begin{equation}
    -\frac{1}{n}\sum_{i = 1}^{n}\log{p_{g'}(x_i|\mathbf{x}^{i-1})}=-\frac{1}{n}\log p_{g}(\mathbf{x})+ \frac{1}{n}\sum_{i=1}^n\log\frac{p_{g}(x_i|\mathbf{x}^{i-1})}{p_{g'}(x_i|\mathbf{x}^{i-1})}. \label{Proof_LLR_Convergence}
\end{equation}
Clearly, the first term in the above sum converges to $H(g)$. To show the convergence of the second term, let $Z_i=\log\frac{p_{g}\paren{x_i|\mathbf{x}^{i-1}}}{p_{g'}\paren{x_i|\mathbf{x}^{i-1}}}$. Notice that for any PFSA $g$ in M2 and for $1\leq i\leq n$, $\mathbf{p}_{g}(\mathbf{x}^{i})$ equals $[1, 0]$ for all $\mathbf{x}^{i}$ with $x_{i} = 0$, and to $[0, 1]$ for all $\mathbf{x}^{i}$ with $x_{i} = 1$, and hence  the process $\{Z_i\}_{i=1}^n$ is a Markov process. Let $\mathcal{Z}^0$ and $\mathcal{Z}^1$ denote the set of indices $i$ such that $x_{i-1} = 0$ and $x_{i - 1} = 1$, respectively. Then we have
\begin{equation}\label{Proof_PartialSum_Z}
    \frac{1}{n}\sum_{i=1}^nZ_i=\frac{1}{n}\sum_{i\in \mathcal{Z}_0}Z_i + \frac{1}{n}\sum_{i\in \mathcal{Z}_1}Z_i.
\end{equation}
It is straightforward to show that for all $i\in \Z_0$ 
$$Z_i=1_{\{x_i=0\}}\log\frac{\mu_g}{\mu_{g'}} + 1_{\{x_i=1\}}\log\frac{\bar{\mu}_g}{\bar{\mu}_{g'}},$$
and for all $i\in \Z_1$
$$Z_i=1_{\{x_i=0\}}\log\frac{\nu_g}{\nu_{g'}} + 1_{\{x_i=1\}}\log\frac{\bar{\nu}_g}{\bar{\nu}_{g'}}.$$
It follows from \eqref{Proof_PartialSum_Z} that 
\begin{eqnarray*}
\frac{1}{n}\sum_{i=1}^nZ_i&=&\frac{1}{n}\paren{\log\frac{\mu_g}{\mu_{g'}}}\sum_{i=1}^n1_{\{x_{i-1}=0,x_i=0\}} + \frac{1}{n} \paren{\log\frac{\bar{\mu}_g}{\bar{\mu}_{g'}}}\sum_{i=1}^n1_{\{x_{i-1}=0,x_i=1\}} + \frac{1}{n}\paren{\log\frac{\nu_g}{\nu_{g'}}}\sum_{i=1}^n1_{\{x_{i-1}=1,x_i=0\}}\\
&&+ \frac{1}{n}\paren{\log\frac{\bar{\nu}_g}{\bar{\nu}_{g'}}}\sum_{i=1}^n1_{\{x_{i-1}=1,x_i=1\}}\\
&\stackrel{n\to \infty}\longrightarrow& 
\mathbf{p}_g(0) \paren{\mu_g \log\frac{\mu_g}{\mu_{g'}} + \bar{\mu}_g \log\frac{\bar{\mu}_g}{\bar{\mu}_{g'}}}+ \mathbf{p}_g(1) \paren{\nu_g \log\frac{\nu_g}{\nu_{g'}} + \bar{\nu}_g \log\frac{\bar{\nu}_g}{\bar{\nu}_{g'}}}.\qedhere
\end{eqnarray*}
\end{proof}
For ease of presentation, we define
\[
    L\paren{g', \mathbf{x}^n\leftarrow g} \coloneqq -\frac{1}{n}\sum_{i = 1}^{n}\log{p_{g'}(x_i|\mathbf{x}^{i-1})}.
\]
When the generating machine $g$ is not known, we use $L\paren{g', \mathbf{x}^n}$ to identify likelihood of $g'$ generating $x$.

\section{Algorithm and simulation}\label{Alg&Simulation}

\subsection{Decoding}\label{subsec:Decoding}
In this and the following section, we assume that we have a set of PFSAs $\G = \set{g_1,\dots, g_{|\M|}}$, with $g_i\in \text{M2}$ for all $i$. We will briefly discuss heuristics on how to generate a set of PFSAs that are good for tamper detecting and decoding in Section\ref{subsec:Design}.

We saw in Theorem \ref{thm:LLHConvergenceM2} that
\begin{equation}\label{eq:goodSeparationPFSAs}
    L\paren{g_j(\delta), \mathbf{x}^{n}\leftarrow g_i(\delta)} \rightarrow H(g_i(\delta)) +D_{\textrm{KL}}\parenBar{g_i(\delta)}{g_j(\delta)},
\end{equation}
which motivates the following definition for the decoding function in Fig.~\ref{Fig1} \[
    \psi(\mathbf{x}) = \arg\min_{m\in\M}L\paren{g_m(\delta), \mathbf{x}^{n}}.
\]
We apply this decoding strategy in Fig.~\ref{fig:pointsAndErrorRate} when $\delta=.2$ and two different message sets with $|\M|=10$ or $|\M|=20$.   
\begin{figure}[ht]
    \centering
    \subfigure[]{\includegraphics[scale=.34]{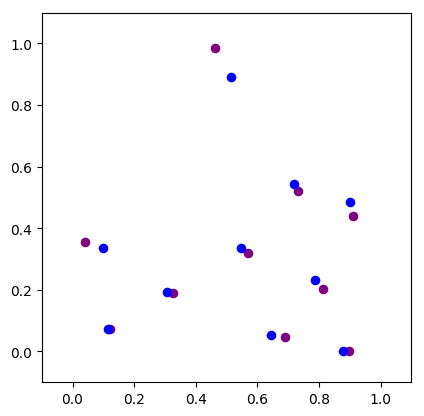}}
    \subfigure[]{\includegraphics[scale=.34]{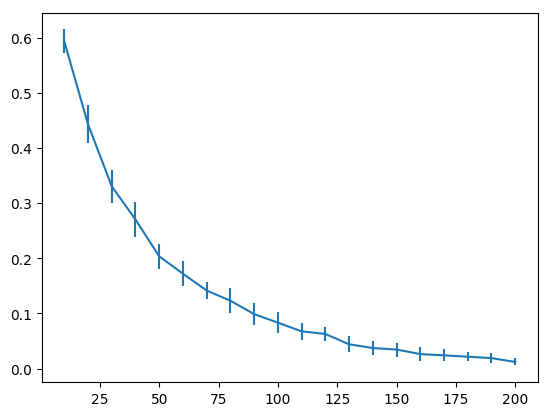}}\hspace{.5cm}
    \subfigure[]{\includegraphics[scale=.34]{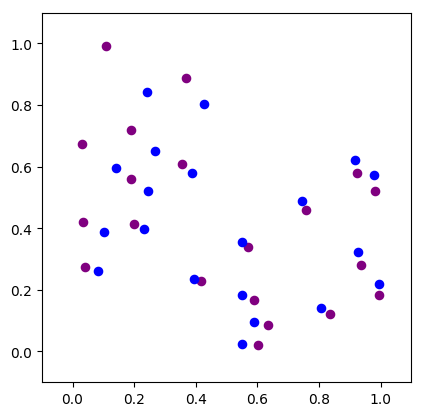}}
    \subfigure[]{\includegraphics[scale=.34]{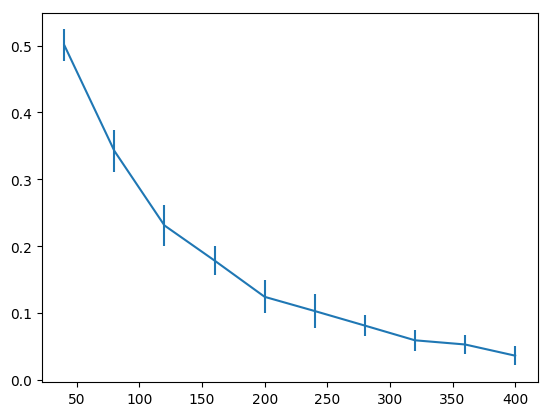}}
    \caption{(a) and (c) shows $10$ PFSAs and $20$ PFSAs in the parameter space, respectively, with purple dots for the $g_m$'s and blue dots for $g_m(.2)$'s. Error rates for input sequences of length $10$ to $200$ for the $10$ messages, and for input sequences of length $40$, $400$ for the $20$ messages are showed in (b) and (d), respectively. The results are averaged over $20$ and $10$ re-runs.}
    \label{fig:pointsAndErrorRate}
\end{figure}

\subsection{Tamper detecting}\label{subsec:DetectTampering}
We assume that active eavesdropper tampers the channel in such a way that $\delta'-\delta>\eta$ with some $\eta\geq0$. Following Theorems  \ref{thm:MonotonicityDeltaM2} and \ref{thm:LLHConvergenceM2}, we get
\begin{align}\label{eq:goodSeparationDeltas}
    L\paren{g_j(\delta), \mathbf{x}\leftarrow g_i(\delta')} &\rightarrow H(g_i(\delta')) + D_{\textrm{KL}}\parenBar{g_i(\delta')}{g_j(\delta)}\nonumber\\
    &\geq H(g_i(\delta)),
\end{align}
where the inequality is due to  Theorem~\ref{thm:MonotonicityDeltaM2}. Hence, tampering the channel results in an increase in the likelihood. This leads to our temper detecting procedure detailed in Algorithm~\ref{Algorithm1}.

\begin{algorithm}
    \SetKwInOut{Input}{input}
    \SetKwInOut{Output}{output}
    \Input{$\set{g_m}_{m\in\M}$, $\mathbf{x}_1, \dots, \mathbf{x}_k$, $\delta, \eta$, $\varepsilon$}
    \Output{$T$ with $T = 0$ if no tampering, $1$ if otherwise}
    $H_0 = \paren{H(g_m(\delta))}_{m\in\M}$\;
    $H_1 = \paren{H(g_m(\delta + \eta))}_{m\in\M}$\;
    $D = H_1 - H_0$\;
     $v = 0$\tcc*[r]{the weighted vote}
    \For{$i = 1,\dots, k$}{
        $d = \arg\min_{m\in\M}L\paren{g_m(\delta), \mathbf{y}_i}$\;
        $e = L\paren{g_d(\delta), \mathbf{y}_i}$\;
        \If{$e - H\paren{g_d(\delta)} > \varepsilon \cdot D[d]$}{
            $v = v + 1 \cdot D[d]$\;
        }
    }
    $S = \sum_{m\in\M}D[m]$\;
    \eIf{$v/(S \cdot k) > 0.5$}{
        \KwRet $T = 1$\;
    }{
        \KwRet $T = 0$\;
    }
    \caption{Tampering detection}
    \label{Algorithm1}
\end{algorithm}
\begin{table}[ht]
    \centering
    \begin{tabular}{r|rrr|rrr|rrr|}
            \cline{2-10}
            & \multicolumn{3}{c}{$\varepsilon = 0$}& \multicolumn{3}{|c|}{$\varepsilon = 0.05$}& \multicolumn{3}{c|}{$\varepsilon = 0.10$}\\
            \cline{1-10}
        $50$  & $.16$ & $.12$ & $.28$ & $.26$ & $.08$ & $.34$ & $.24$ & $.18$ & $.32$\\
        $100$ &   $0$ & $.26$ & $.26$ & $.06$ & $.20$ & $.26$ & $.08$ & $.08$ & $.16$\\
        $150$ &   $0$ & $.16$ & $.16$ & $.02$ & $.22$ & $.14$ &   $0$ & $.08$ & $.08$\\
        $200$ &   $0$ & $.18$ & $.18$ &   $0$ & $.20$ & $.20$ &   $0$ & $.06$ & $.06$\\
        \cline{1-10}
            & \multicolumn{3}{c}{$\varepsilon = 0.15$}& \multicolumn{3}{|c|}{$\varepsilon = 0.20$}& \multicolumn{3}{c|}{$\varepsilon = 0.25$}\\
        \cline{1-10}
        $50$  & $.36$ & $.02$ & $.38$ & $.32$ & $.02$ & $.34$ & $.26$ & $.10$ & $.36$\\
        $100$ & $.08$ & $.08$ & $.16$ & $.08$ & $.02$ & $.10$ & $.14$ & $.04$ & $.18$\\
        $150$ &   $0$ & $.08$ & $.08$ & $.02$ & $.02$ & $.04$ & $.02$ & $.02$ & $.04$\\
        $200$ &   $0$ &   $0$ &   $0$ &   $0$ & $.04$ & $.04$ &   $0$ & $.02$ & $.02$\\
        \cline{1-10}
    \end{tabular}
    \caption{The table above records the error rates of tamper detection algorithm for sending $10$ messages through a channel with deletion probability $\delta = .2$. We generate $50$ test sets containing $k = 200$ sequences, with $20$ for each message. We assign randomly whether a particular test set will be tampered or not. For simplicity, if a test set is tampered it will have a fixed deletion probability $\delta = .3$. We run the algorithm for input sequence of length $50, 100, 150$, and $200$, and for $\varepsilon = 0, .05, .10, .15, .20, .25$. For each block, the first column is the rate of failing to detect a tampering, and the second column is the rate of false alarm of a tampering, and the last column is the sum of two error rates. We can see that with increased cutoff value $\varepsilon$, we have significantly fewer false alarms without too much increase in the rate of failing to detect a true tampering.}
    \label{tab:tamperingDetectionError}
\end{table}
\subsection{Generate machines with good separation}\label{subsec:Design}
For fixed number of messages, we need to choose a set of M2 PFSAs with the best decoding and tamper detection performance. It is important to indicate that (1) decoding error will be significantly lowered by increasing $D(g_i\|g_j)$ according to \eqref{eq:goodSeparationPFSAs}, and (2) the tampering detection error will be improved by making sure $|H(g(\delta)) - H(g(\delta'))|$ is large for $\delta'-\delta\geq \eta$, according to \eqref{eq:goodSeparationDeltas}. However, there is a trade-off here -- to increase pairwise KL divergence, we want the machines to be spread more evenly in the parameter space while, according to Theorem \ref{thm:MonotonicityDeltaM2}, to increase $H(g(\delta')) - H(g(\delta)$, we need the machines to stay away from being single-state, i.e.~away from the $\mu = \nu$ line. 

Here, we describe briefly how we design $\G$ for experiment in Fig.~\ref{fig:pointsAndErrorRate}.  As a naive way, we start off with $|\M|$ randomly generated $\mu$'s in $(0, 1)$, and for each of them we generate $\nu$ in the following way: if $\mu > .5$, then we choose a $\nu$ randomly in $(0, \mu - .2)$, and if $\mu \leq .5$, in $(\mu + .2, 1.)$. Then, we use a hill-climbing algorithm to maximize minimum pairwise averaged KL divergence, $.5\paren{D_{\textrm{KL}}\parenBar{g_1}{g_2} + D_{\textrm{KL}}\parenBar{g_2}{g_1}}$, between all pair machines. Let $\sigma$ be step size, for a pair $g_{(\mu_1, \nu_1)}$ and $g_{(\mu_2, \nu_2)}$ with minimum averaged KL divergence, we search the eight neighboring points, $(\mu_i\pm\sigma, \nu_i)$ and $(\mu, \nu_i\pm\sigma)$, $i=1, 2$, for improvement. We exit the search when there is no improvement to be found.
\section{Conclusion and future work}
In this paper, we developed a new information-theoretic coding scheme for information transfer over a public deletion channel, subject to an active eavesdropper (aka jammer). Our coding scheme is based on probabilistic finite-state automata (PFSA) and is proved to have (1) semi-universal property, in a sense that codebook need not be available at the decoder, (2) small error probability when decoding messages, and (3) tamper-free property, which alarms the decoder about possible tampering of the channel.  To the best of our knowledge, exploiting PFSA's in a secure and reliable information-theoretic communication model is very new, yet very insightful. Promising results in both theoretical and experimental aspects of this work lead to several research directions: 
\begin{itemize}
    \item To have an analytically better analysis of error probability, the convergence rate of likelihood in Theorem\ref{thm:LLHConvergenceM2} for general PFSA is needed.
    \item We admit that the space of M2 is too small to have simultaneous vanishing error probability (with small $n$) in message decoding and tamper detecting. To go beyond M2, we need to find an analytic way to compute entropy rate and KL divergence for generalized PFSA.  
\end{itemize}

\bibliographystyle{IEEEtran}
\bibliography{bibliography}

\section*{Appendix}
\subsection{Proof for Theorem \ref{thm:EntropyRate}} \label{Appendix_Entropy}
Following the standard notation in information theory, we use $X^n$ to denote a random vector $(X_1, \dots, X_n)$ generated from a PFSA $g$ and $H(X^n)$ to denote the entropy its entropy, that is $H(X^n)=H_n(g)$. We can similarly define the conditional entropy $H(X_n|X^{n-1})$. It is shown in \cite{Cover:EIT} that $\lim_{n\rightarrow\infty}\frac{1}{n}H\paren{X^{n}}=\lim_{n\rightarrow\infty}H\paren{X_{n}|X^{n-1}}$ for any stationary processes $\{X_n\}_{n=1}^\infty$. In order to compute the entropy rate, we can therefore focus on the latter limit. Let $S\sim \mathbf{p}$ denote a random variable indicating the initial state of the PFSA. We have
\begin{align*}
    H\paren{X_n|X^{n-1}} &= H\paren{X^{n}} - H\paren{X^{n-1}}\\
    &=\bracket{H\paren{X^{n}, S} - H\paren{S|X^{n}}}- \bracket{H\paren{X^{n-1}, S} - H\paren{S|X^{n-1}}}\\
    &=\bracket{H\paren{X^{n}, S} - H\paren{X^{n-1}, S}} + \bracket{H\paren{S|X^{n-1}} - H\paren{S|X^{n}}}\\
    &=\bracket{H\paren{X^{n}|S} + H(S) - H\paren{X^{n-1}|S} - H(S)} + \bracket{H\paren{S|X^{n-1}} - H\paren{S|X^{n}}}\\
    &=\underbrace{H\paren{x_n|S,X^{n-1}}}_{\eqqcolon A_n} + \underbrace{\left[H\paren{S|X^{n-1}} - H\paren{S|X^{n}}\right]}_{\eqqcolon B_n}.
\end{align*}
Note that for any $N\geq 1$
\[
    \sum_{n = 1}^{N}B_n = \sum_{n = 1}^{N}H\paren{S|X^{n-1}} - H\paren{S|X^{n}} = H(S) - H\paren{S|X^{N}} \leq H(S) = H\paren{\mathbf{p}}.
\]
\begin{figure}[t]
    \centering
    \subfigure[$g_1\in$M2]{\includegraphics[scale=1.]{M2.pdf}} \hspace{1.3cm}
    \subfigure[$g_2\in$M2]{\includegraphics[scale=1.]{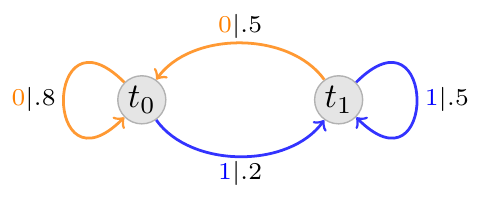}}
    
    \subfigure[$g^{*}_{\textrm{c}}\parenBar{g_1}{g_2}$]{\includegraphics[scale=.7]{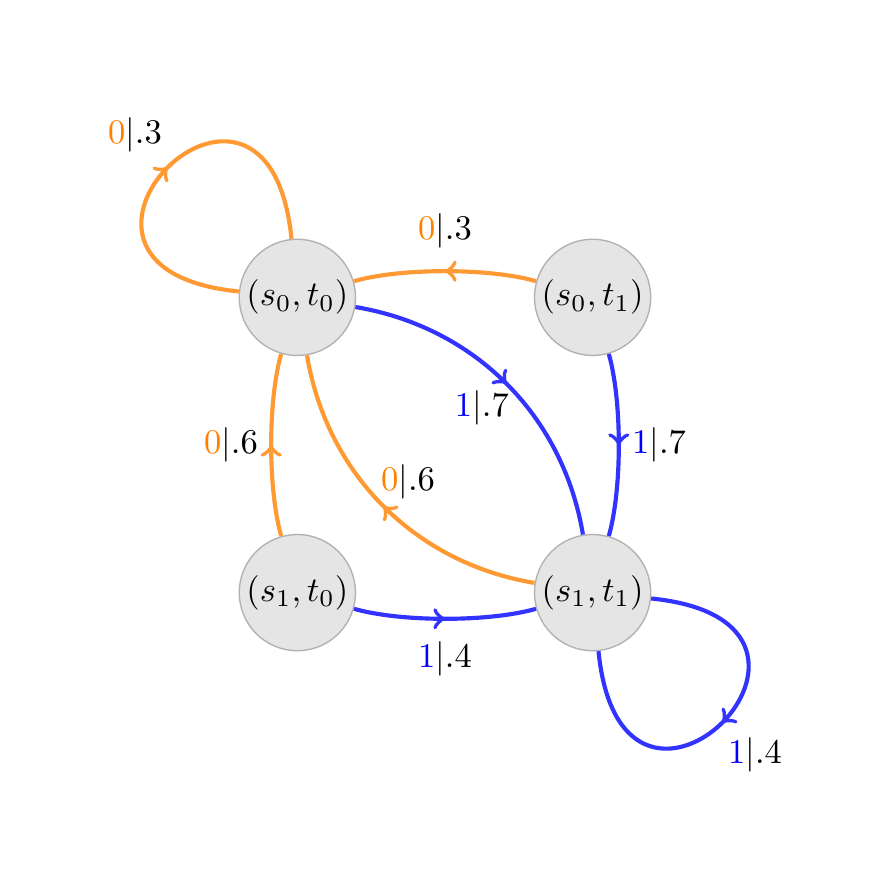}}
    \subfigure[strongly connected component of $g^{*}_{\textrm{c}}\parenBar{g_1}{g_2}$]{\includegraphics[scale=.7]{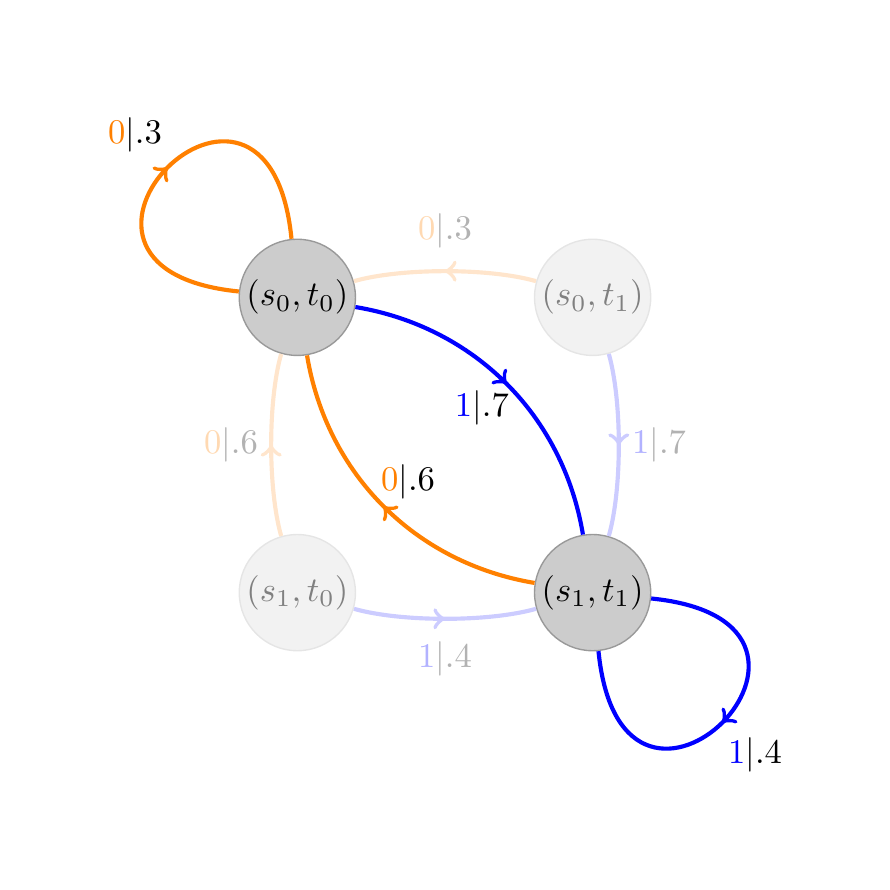}}
    \caption{The example above shows that the $g^{*}_c$ of two strongly connected PFSAs may \emph{not} remain strongly connected. We can see that in this case, $g_{\textrm{c}}\parenBar{g_1}{g_2}$ is equal to $g_1$.}
    \label{fig:SyncCompM2&M2}
\end{figure}
Since $B_n$ is nonnegative for each $n$ and $\sum_{n=1}^N B_n$ is bounded from above, it follows that $\lim_{n\rightarrow\infty}B_n = 0$. It remains to analyze $A_n$. 
Notice that the state at time $n$ is a deterministic function of $S$ and $X^{n-1}$ (that is $\mathsf{T}(s, X^{n-1})$) and hence we can write
\begin{align*}
    H\paren{X_n|S,X^{n-1}} &= \sum_{s'\in\S}H\paren{\widetilde{P}_{s',\cdot}}\Pr\set{\mathsf{T}(S, X^{n-1})=s'}.
\end{align*}
By induction, we have for any $s'\in \S$
\begin{align*}
    \Pr\set{\mathsf{T}(S, X^{n-1})=s'} &= \sum_{x\in\X}\sum_{s''\in\S}\Pr\set{\mathsf{T}(s, X^{n-2})=s''}\paren{\Gamma_{x}}_{s'', s'}\\
    &=\sum_{s''\in\S}\Pr\set{\mathsf{T}(s, X^{n-2})=s''}P_{s'', s'},
\end{align*}
and hence
\[
    \paren{\Pr\set{\mathsf{T}(s, X^{n-1})=s}}_{s\in\S} = \paren{\Pr\set{\mathsf{T}(s, X^{n-2})=s}}_{s\in\S}P = \cdots = \mathbf{p}P^{n-1} = \mathbf{p}.\qedhere
\]

\begin{figure}[h]
    \centering
    \subfigure[$g_3$]{\includegraphics[trim=0 -1cm 0 0, clip,scale=1.]{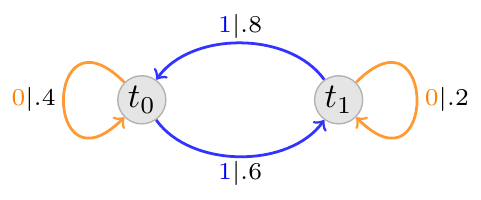}}\hspace{1.cm}
    \subfigure[$g_{\textrm{c}}\parenBar{g_1}{g_3}$]{\includegraphics[scale=.7]{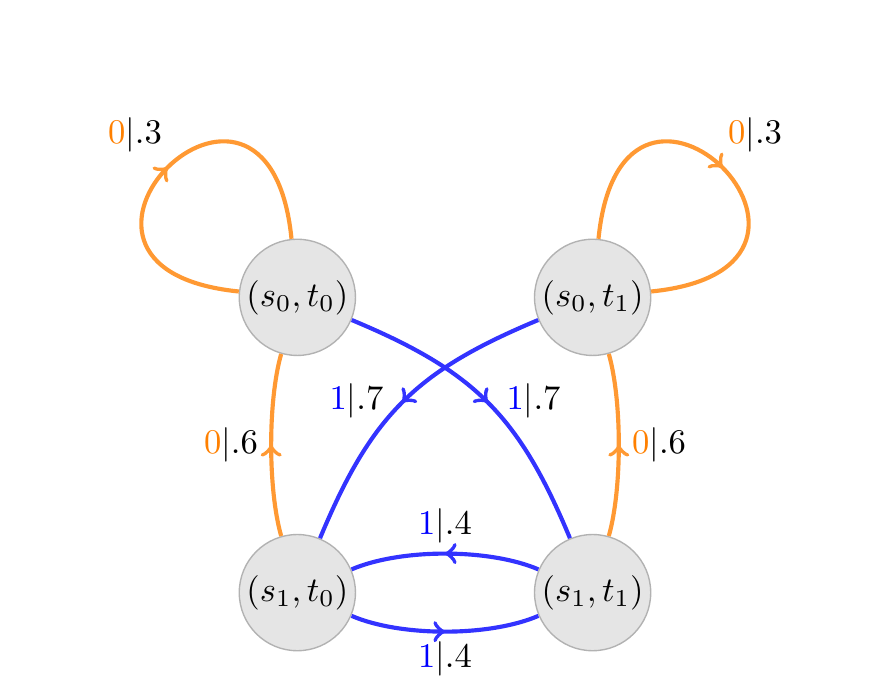}}
    \caption{$g_{\textrm{c}}\parenBar{g_1}{g_3}$ is strongly connected. The stationary distribution of $g_{\textrm{c}}\parenBar{g_1}{g_3}$ is $(.231, .231,  .269, .269)$, while the stationary distribution of $g_1$ is $(.462, .538)$, both rounded to 3 decimal places.}
    \label{fig:SyncCompM2&S}
\end{figure}

\begin{figure}[h]
    \centering
    \subfigure[$g_4$]{\includegraphics[trim=0 -.5cm 0 0, clip, scale=1.]{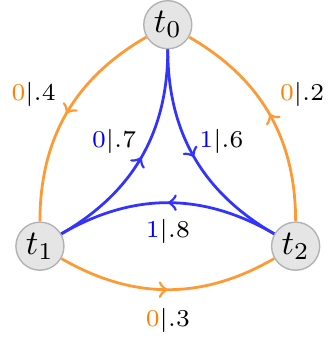}} \hspace{1.3cm}
    \subfigure[$g_{\textrm{c}}\parenBar{g_1}{g_4}$]{\includegraphics[scale=.7]{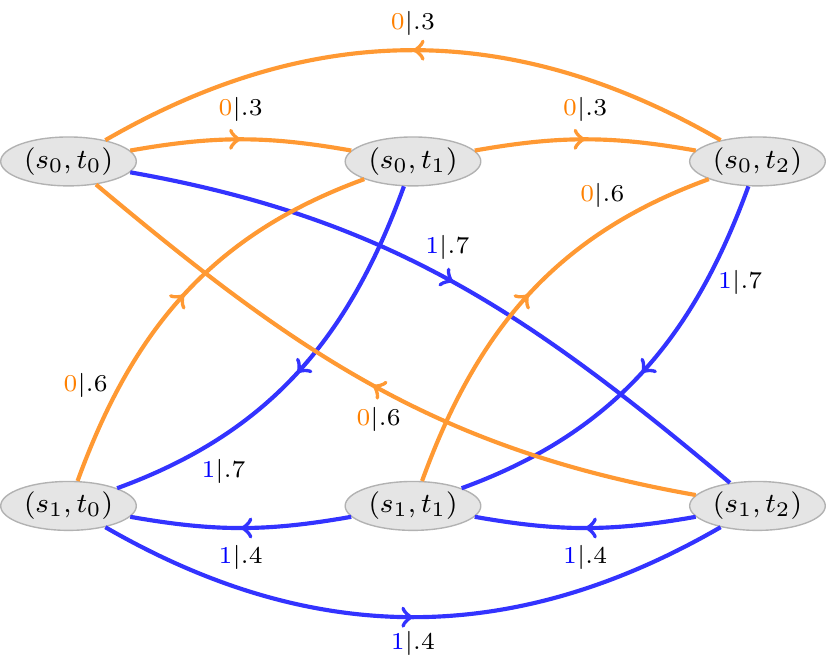}}
    \caption{$g_{\textrm{c}}\parenBar{g_1}{g_4}$ is strongly connected. The stationary distribution of $g_{\textrm{c}}\parenBar{g_1}{g_4} = (.154, .154, .154, .179, .179, .179)$, while the stationary distribution of $g_1$ is $(.462, .538)$, both rounded to 3 decimal places.}
    \label{fig:SyncCompM2&T}
\end{figure}

\begin{figure}
    \centering
    \subfigure[$g_5$]{\includegraphics[scale=1.]{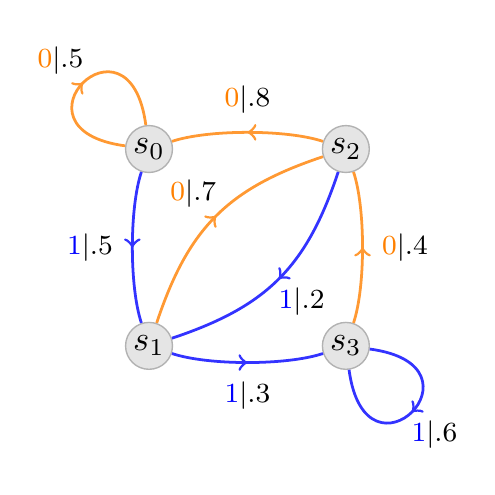}} \hspace{1.cm}
    \subfigure[$g^{*}_{\textrm{c}}\parenBar{g_5}{g_2}$]{\includegraphics[scale=.6]{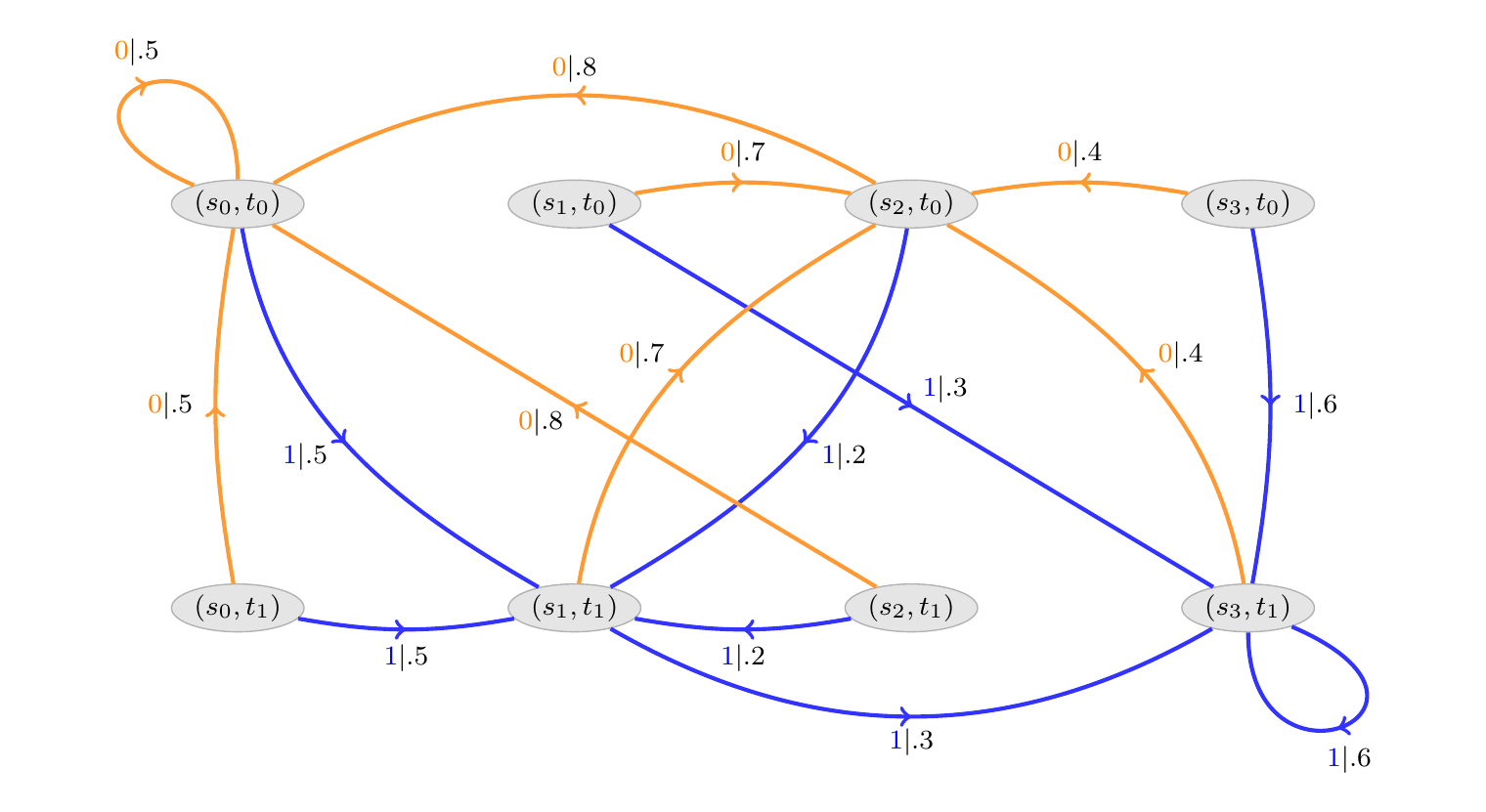}}
    \subfigure[strongly connected component of $g^{*}_{\textrm{c}}\parenBar{g_5}{g_2}$]{\includegraphics[scale=.6]{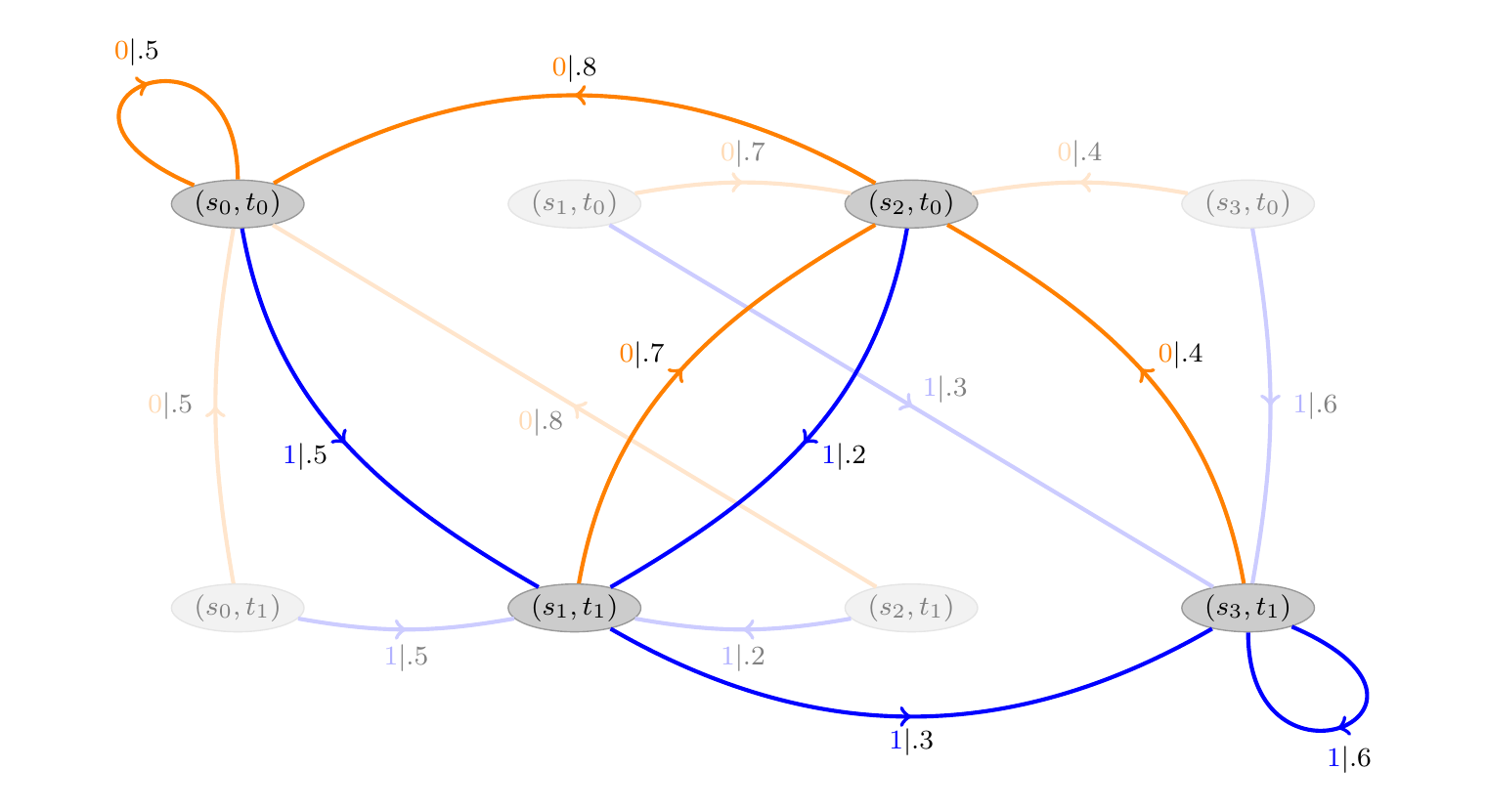}}
    \caption{$g_{\textrm{c}}\parenBar{g_5}{g_2}$ is the strongly connected component of $g^{*}_{\textrm{c}}\parenBar{g_5}{g_2}$ and it is equal to $g_5$.}
    \label{fig:SyncCompM4&M2}
\end{figure}

\subsection{Proof for Theorem \ref{thm:KLDivergence}} \label{Appendix:KL}
Before we can prove Theorem \ref{thm:KLDivergence}, we first study synchronous compositions in more detail. Specifically, we shall show that $\mathbf{p}_{g_c}(\mathbf{x})$ is independent of the choice of absorbing strongly connected component in $g^{*}_{\textrm{c}}\parenBar{g_1}{g_2}$. Essentially,  $g_{\textrm{c}}\parenBar{g_1}{g_2}$ is equivalent (to be defined later) to $g_1$, which is key to the usage of synchronous composition in the proof of Theorem \ref{thm:KLDivergence}.
\begin{definition}\label{def:T_s}
Let $g_1=\paren{\S, \X, \mathsf{T}_1, \mathsf{P}_1}$ and $g_2=\paren{\T, \X, \mathsf{T}_2, \mathsf{P}_2}$ be two PFSAs with the same alphabet and let $g_{\textrm{c}}\parenBar{g_1}{g_2}$ be the synchronous composition of $g_1$ and $g_2$. Suppose that the state space of $g_{\textrm{c}}\parenBar{g_1}{g_2}$ is $\U\subset\S\times\T$. We then define $\T_{s} = \set{t\in\T: (s, t)\in\U}$.
\end{definition}
We provided several examples of synchronous compositions in Figs.~\ref{fig:SyncCompM2&M2} to \ref{fig:SyncCompM4&M2}. We note that, in Fig.~\ref{fig:SyncCompM2&S} and \ref{fig:SyncCompM2&T}, the compositions $g^*_{\textrm{c}}$ are naturally strongly connected, while those in Fig.~\ref{fig:SyncCompM2&M2} and \ref{fig:SyncCompM4&M2} are not. For $g_{\textrm{c}}\parenBar{g_1}{g_2}$ in Fig.~\ref{fig:SyncCompM2&M2}, we have $\T_{s_0}=\set{t_0}$ and $\T_{s_1} = \set{t_1}$, and for $g_{\textrm{c}}\parenBar{g_5}{g_2}$ in Fig.~\ref{fig:SyncCompM4&M2}, we have $\T_{s_0}=\set{t_0}$, $\T_{s_1} = \set{t_1}$, $\T_{s_2}=\set{t_0}$, and $\T_{s_3} = \set{t_1}$. 

\begin{proposition}\label{prop:stationaryOfComp}
Let $g_{\textrm{c}} = g_{\textrm{c}}\parenBar{g_1}{g_2}$ be any absorbing strongly connected component of $g^{*}_{\textrm{c}}\parenBar{g_1}{g_2}$ and let $\mathbf{p}_{g_{\textrm{c}}}$ be its stationary distribution. Then we have $\sum_{t\in\T_s}\paren{\mathbf{p}_{g_{\textrm{c}}}}_{(s, t)} = \paren{\mathbf{p}_{g_1}}_{s}$.
\end{proposition}
\begin{proof}
For any fixed initial state $(s, t)$ and any sequence of symbols $\mathbf{x}^{n}\in\X^n$, consider the sequence of states of the synchronous composition
\[
(s, t), \paren{\mathsf{T}_1\paren{s,x_{1}}, \mathsf{T}_2\paren{t, x_{1}}}, \dots, \paren{\mathsf{T}_1\paren{s,\mathbf{x}^{n}}, \mathsf{T}_2\paren{t, \mathbf{x}^{n}}}.
\]
Let $n_{s',t'}$ be the number of indices $i = 1, \dots, n$ such that $\paren{\mathsf{T}_1\paren{s,\mathbf{x}^{i}}, \mathsf{T}_2\paren{t, \mathbf{x}^{i}}} = (s', t')$. Since the associated stochastic process on states induced by $g_{\textrm{c}}$ is stationary and ergodic, we have $\frac{n_{s',t'}}{n} \to \paren{\mathbf{p}_{g_{\textrm{c}}}}_{(s',t')}$ as $n\to \infty$ in probability. 
Consequently, $$\sum_{t'\in \T_s}\frac{n_{s',t'}}{n} \to \sum_{t'\in \T_s}\paren{\mathbf{p}_{g_{\textrm{c}}}}_{(s',t')}.$$
Noticing that the left-hand side converges to $ \paren{\mathbf{p}_{g_1}}_{s}$, we obtain the result.
\end{proof}
Figs.~\ref{fig:SyncCompM2&S} and \ref{fig:SyncCompM2&T} provide examples of the proposition above. 
\begin{theorem}\label{Thm:Equivalence}
Let $g_{\textrm{c}} = g_{\textrm{c}}\parenBar{g_1}{g_2}$ be any absorbing strongly connected component of $g^{*}_{\textrm{c}}\parenBar{g_1}{g_2}$. Then we have $g_{\textrm{c}}$ is equivalent to $g_1$, in the sense that $p_{g_{\textrm{c}}}\paren{\mathbf{x}} = p_{g_1}\paren{\mathbf{x}}$ for $\mathbf{x}\in\X^{*}$.
\end{theorem}
\begin{proof}
We first show
\begin{equation}\label{eq:equivalentTo(g1)}
    \sum_{t\in\T_s}p_{g_{\textrm{c}}}\paren{\mathbf{x}|(s, t)}\paren{\mathbf{p}_{g_{\textrm{c}}}}_{(s, t)} = p_{g_1}\paren{\mathbf{x}|s}\paren{\mathbf{p}_{g_1}}_{s}
\end{equation}
by induction on the length of $\mathbf{x}$. We first note that the base case in which $\mathbf{x}$ is the empty sequence is given by Proposition \ref{prop:stationaryOfComp}. Now assume that \eqref{eq:equivalentTo(g1)} holds for $\abs{\mathbf{x}} = n$. Follow the notation as in Definition \ref{def:T_s}, we have for sequence $\mathbf{x}x$
\begin{align*}
    \sum_{t\in\T_s}p_{g_{\textrm{c}}}\paren{\mathbf{x}x|(s, t)}\paren{\mathbf{p}_{g_{\textrm{c}}}}_{(s, t)} &= \sum_{t\in\T_s}p_{g_{\textrm{c}}}\paren{\mathbf{x}|(s, t)}p_{g_{\textrm{c}}}\paren{x|\mathbf{x}, (s, t)}\paren{\mathbf{p}_{g_{\textrm{c}}}}_{(s, t)}\\
    &=\sum_{t\in\T_s}p_{g_{\textrm{c}}}\paren{\mathbf{x}|(s, t)}\mathsf{P}_1\paren{\mathsf{T}_1(s, \mathbf{x}),x}\paren{\mathbf{p}_{g_{\textrm{c}}}}_{(s, t)}\\
    &=\paren{\sum_{t\in\T_s}p_{g_{\textrm{c}}}\paren{\mathbf{x}|(s, t)}\paren{\mathbf{p}_{g_{\textrm{c}}}}_{(s, t)}}\mathsf{P}_1\paren{\mathsf{T}_1(s, \mathbf{x}),x}\\
    &\stackrel{(a)}{=}p_{g_1}\paren{\mathbf{x}|s}\paren{\mathbf{p}_{g_1}}_{s}\mathsf{P}_1\paren{\mathsf{T}_1(s, \mathbf{x}),x}\\
    &=p_{g_1}\paren{\mathbf{x}x|s}\paren{\mathbf{p}_{g_1}}_{s},
\end{align*}
where equality in $(a)$ follows from the induction hypothesis. Now we can write
\begin{align*}
    p_{g_{\textrm{c}}}\paren{\mathbf{x}}&=\sum_{s\in\S}\sum_{t\in\T_s}p_{g_{\textrm{c}}}\paren{\mathbf{x}|(s, t)}\paren{\mathbf{p}_{g_{\textrm{c}}}}_{(s, t)}= \sum_{s\in\S}p_{g_1}\paren{\mathbf{x}|s}\paren{\mathbf{p}_{g_1}}_{s} = p_{g_1}\paren{\mathbf{x}},
\end{align*}
from which the result follows.
\end{proof}

\begin{proof}[Proof for Theorem \ref{thm:KLDivergence}]

We use the same notation as in Appendix~\ref{Appendix_Entropy}. We start the proof by defining two distributions on the Cartesian product $\S\times\T\times\X^{n}$. Let 
\begin{align*}
    g_{12} \coloneqq g_{\textrm{c}}\parenBar{g_1}{g_2},\quad g_{21} \coloneqq g_{\textrm{c}}\parenBar{g_2}{g_1},
\end{align*} 
and $\mathbf{p}_{12}$ and  $\mathbf{p}_{21}$ be the stationary distributions of $g_{12}$ and $g_{21}$, respectively. Here we make sure that we choose the same absorbing strongly connected component for both compositions. We notice that $g_{12}$ and $g_{21}$ induce two distributions $p_{12}$, and $p_{21}$ on $\S\times\T\times\X^{n}$ given by  $p_{12}(s, t, \mathbf{x}^{n-1})=p_{12}(s, t)p_{12}(\mathbf{x}^{n-1}|s, t)$ and $p_{21}(s, t, \mathbf{x}^{n-1})=p_{21}(s, t)p_{21}(\mathbf{x}^{n-1}|s, t)$ where
\begin{align*}
    &p_{12}(s, t) = \paren{\mathbf{p}_{12}}_{(s,t)},\quad p_{21}(s, t) = \paren{\mathbf{p}_{21}}_{(s,t)},\\
    &p_{12}\paren{\mathbf{x}^n | s, t} = p_{g_1}\paren{\mathbf{x}^{n}|s} = \prod_{i = 1}^{n}\mathsf{P}_1\paren{\mathsf{T}_1\paren{s, \mathbf{x}^{i-1}}, x_i},\\
    &p_{21}\paren{\mathbf{x}^n | s, t} = p_{g_2}\paren{\mathbf{x}^{n}|t} = \prod_{i = 1}^{n}\mathsf{P}_2\paren{\mathsf{T}_2\paren{t, \mathbf{x}^{i-1}}, x_i}.
\end{align*}
Letting $p_{12}(\mathbf{x}^n)$ ($p_{12}(\mathbf{x}^{n-1})$) be the marginal of $p_{12}$ over $\X^n$ (resp. $\X^{n-1}$), we can write using the chain rule of KL divergence (see e.g., \cite[Theorem 2.5.3]{Cover:EIT}) that
\begin{align*}
    &D_{\textrm{KL}}\parenBar{p_{12}\paren{X^n}}{{p_{21}\paren{X^n}}} - D_{\textrm{KL}}\parenBar{p_{12}\paren{X^{n-1}}}{{p_{21}\paren{X^{n-1}}}}\\
    =&\bracket{D_{\textrm{KL}}\parenBar{p_{12}\paren{S, T, X^n}}{{p_{21}\paren{S, T, X^n}}} - D_{\textrm{KL}}\parenBar{p_{12}\paren{S, T|X^n}}{{p_{21}\paren{S, T| X^n}}}}\\ &- \bracket{D_{\textrm{KL}}\parenBar{p_{12}\paren{S, T,X^{n-1}}}{{p_{21}\paren{S, T, X^{n-1}}}}- D_{\textrm{KL}}\parenBar{p_{12}\paren{S, T| X^{n-1}}}{{p_{21}\paren{S, T| X^{n-1}}}}}\\
    =&\bracket{D_{\textrm{KL}}\parenBar{p_{12}\paren{S, T, X^{n}}}{{p_{21}\paren{S, T, X^{n}}}} - D_{\textrm{KL}}\parenBar{p_{12}\paren{S, T,X^{n-1}}}{{p_{21}\paren{S, T, X^{n-1}}}}}\\
    &-\bracket{D_{\textrm{KL}}\parenBar{p_{12}\paren{S, T|X^n}}{{p_{21}\paren{S, T| X^n}}} - D_{\textrm{KL}}\parenBar{p_{12}\paren{S, T| X^{n-1}}}{{p_{21}\paren{S, T| X^{n-1}}}}}\\
    =&\underbrace{D_{\textrm{KL}}\parenBar{p_{12}\paren{X_n\left|S, T,X^{n-1}\right.}}{{p_{21}\paren{X_n\left|S, T, X^{n-1}\right.}}}}_{\eqqcolon C_n}\\
    &-\bracket{\underbrace{D_{\textrm{KL}}\parenBar{p_{12}\paren{S, T|X^n}}{{p_{21}\paren{S, T| X^n}}}}_{\eqqcolon D_n} - \underbrace{D_{\textrm{KL}}\parenBar{p_{12}\paren{S, T|X^{n-1}}}{{p_{21}\paren{S, T| X^{n-1}}}}}_{{D_{n-1}}}}.
\end{align*}
We first show that $C_n$ is a constant that equals the desired quantity. Notice that for a fixed initial state $(s, t)\in \S\times \T$ and a fixed sequence $\mathbf{x}^{n-1}\in \X^{n-1}$ we have $\mathsf{T}_{\textrm{c}}((s, t), \mathbf{x}^{n-1}) = \paren{\mathsf{T}_1(s, \mathbf{x}^{n-1}), \mathsf{T}_2(t, \mathbf{x}^{n-1})}$ and hence
\begin{align*}
    C_n &= \sum_{s', t'}D_{\textrm{KL}}\parenBar{\paren{\widetilde{P}_{g_1}}_{s',\cdot}}{\paren{\widetilde{P}_{g_2}}_{t',\cdot}}\cdot p_{12}\set{\mathsf{T}_1\paren{s, \mathbf{x}^n} = s', \mathsf{T}_2\paren{t, \mathbf{x}^n} = t'}\\
    &= \sum_{s', t'}D_{\textrm{KL}}\parenBar{\paren{\widetilde{P}_{g_1}}_{s',\cdot}}{\paren{\widetilde{P}_{g_2}}_{t',\cdot}}\cdot {\mathbf{p}_{12}}_{(s',t')}.
\end{align*}

We next show that $D_n$ converges in probability and in particular $D_n-D_{n-1}\to 0$. For a fixed initial state $(s, t)$ and a sequence $\mathbf{x}^{n}$, consider the sequence of states $s, \mathsf{T}_1\paren{s, \mathbf{x}^1}, \mathsf{T}_1\paren{s, \mathbf{x}^2}, \dots, \mathsf{T}_1\paren{s, \mathbf{x}^n}$, and let $n_{s', x} = n_{s', x}(s)$ denote the number of indices $i$ such that $\mathsf{T}_1(s, \mathbf{x}^{i-1}) = s'$ and $x_i = x$. We have for all $t\in\T_s$ 
\begin{align*}
    p_{12}\paren{\mathbf{x}^{n}|s, t} = \prod_{i = 1}^{n}\mathsf{P}_1\paren{\mathsf{T}_1\paren{s, \mathbf{x}^{i-1}}, x_i} = \prod_{s', x}\mathsf{P}_1\paren{s', x}^{n_{s', x}}
    =2^{\sum_{s', x}n_{s', x}\log\mathsf{P}_1\paren{s', x}} = 2^{n\sum_{s', x}\frac{n_{s', x}}{n}\log\mathsf{P}_1\paren{s', x}}.
\end{align*}
Since the associated stochastic process on states is stationary and ergodic, we have $\frac{n_{s', x}}{n}\rightarrow\paren{\mathbf{p}_{g_1}}_{s'}\mathsf{P}_1(s', x)$ in probability as $n\rightarrow\infty$, and hence $p_{12}\paren{\mathbf{x}^{n}|s, t}\rightarrow{2^{-nH(g_1)}}$ in probability and \emph{independent} of the initial state $s$. This implies $\paren{p_{12}\paren{s, t\left|\mathbf{x}^{n}\right.}}_{(s, t)}$ and $\paren{p_{21}\paren{s, t \left|\mathbf{x}^{n}\right.}}_{(s, t)}$ converge in probability to the stationary distribution $\mathbf{p}_{12}$ and $\mathbf{p}_{21}$, respectively, which shows that $D_n$ converges and hence the theorem follows.
\end{proof}

\end{document}